%% file: main.tex
\documentclass{article}
\usepackage{PRIMEarxiv}

\usepackage[round]{natbib}

\usepackage[utf8]{inputenc} 
\usepackage[T1]{fontenc}    
\usepackage{hyperref}       
\usepackage{url}            
\usepackage{booktabs}       
\usepackage{amsfonts}       
\usepackage{nicefrac}       
\usepackage{microtype}      

\usepackage{amssymb,amsmath,amsthm}
\usepackage[dvipsnames]{xcolor}
\usepackage{graphicx}
\usepackage{caption}
\usepackage{subcaption}

\usepackage{dsfont} 
\usepackage[ruled,linesnumbered,vlined]{algorithm2e}
\usepackage{algorithmic}
\usepackage{amsmath}
\usepackage{epstopdf}
\usepackage{wrapfig}
\usepackage{enumitem}
\usepackage{tikz}
\newcommand{\tikzxmark}{%
\tikz[scale=0.23] {
    \draw[line width=0.7,line cap=round] (0,0) to [bend left=6] (1,1);
    \draw[line width=0.7,line cap=round] (0.2,0.95) to [bend right=3] (0.8,0.05);
}}

\newtheorem*{proposition*}{Proposition}
\newtheorem*{theorem*}{Theorem}
\newtheorem{assumption}{Assumption}
\newtheorem{definition}{Definition}

\newtheorem{lemma}{Lemma}
\newtheorem{theorem}{Theorem}
\newtheorem{corollary}{Corollary}

\newcommand{\PP}{\mathbb{P}}

\newcommand{\EE}{\mathbb{E}}
\newcommand{\II}{\mathbb{I}}
\newcommand{\RR}{\mathbb{R}}
\newcommand{\As}{\mathcal{A}}

\newcommand{\Fs}{\mathcal{F}}
\newcommand{\Cs}{\mathcal{C}}
\newcommand{\Ss}{\mathcal{S}}

\newcommand{\Ps}{\mathcal{P}}

\newcommand{\Os}{\mathcal{O}}

\newcommand{\Es}{\mathcal{E}}

\newcommand{\bbone}{\mathds{1}}
\newcommand{\hGamma}{\hat{\Gamma}}

\DeclareMathOperator*{\argmax}{argmax}

\newcommand{\azedit}[1]{{\color{black}#1}}

\newcommand{\softmax}[2]{\text{softmax}_{#1}(#2)}

\newcommand{\net}{\text{net}}

\title{Oracle-free Reinforcement Learning in Mean-Field Games along a Single Sample Path
\thanks{\noindent Research of the first and fourth authors was supported in part by the Air Force Office of Scientific Research (AFOSR) Grant FA9550-19-1-0353.}
}

\author{%
  Muhammad~Aneeq~uz~Zaman \\
  Coordinated Science Laboratory\\
  University of Illinois at Urbana-Champaign\\
  Urbana IL 61801-2307 \\
  \texttt{mazaman2@illinois.edu} \\
  \And
  Alec~Koppel \\
  J.P. Morgan AI Research  \\
  383 Madison Ave\\
  New York, NY 10017\\
  \texttt{alec.koppel@jpmchase.com} \\
  \And
  Sujay~Bhatt\\
  J.P. Morgan AI Research  \\
  383 Madison Ave\\
  New York, NY 10017\\
  \texttt{sujay.bhatt@jpmchase.com} \\
  \And
  Tamer Ba{\c s}ar \\
  Coordinated Science Laboratory\\
  University of Illinois at Urbana-Champaign\\
  Urbana IL 61801-2307 \\
  \texttt{basar1@illinois.edu} \\
}

\begin{document}

\maketitle






\begin{abstract}
We consider online reinforcement learning in Mean-Field Games (MFGs). Unlike traditional approaches, we alleviate the need for a mean-field oracle by developing an algorithm that approximates the Mean-Field Equilibrium (MFE) using the single sample path of the generic agent. We call this {\it Sandbox Learning}, as it can be used as a warm-start for any agent learning in a multi-agent non-cooperative setting. We adopt a two time-scale approach in which an online fixed-point recursion for the mean-field operates on a slower time-scale, in tandem with a control policy update on a faster time-scale for the generic agent. Given that the underlying Markov Decision Process (MDP) of the agent is communicating, we provide finite sample convergence guarantees in terms of convergence of the mean-field and control policy to the mean-field equilibrium. The sample complexity of the Sandbox learning algorithm is $\Os(\epsilon^{-4})$ where $\epsilon$ is the MFE approximation error. This is similar to works which assume access to oracle. Finally, we empirically demonstrate the effectiveness of the sandbox learning algorithm in diverse scenarios, including those where the MDP does not necessarily have a single communicating class.
\end{abstract}

\section{INTRODUCTION}

\input{Sec_Intro}

\section{FORMULATION \& BACKGROUND}
\label{sec:form}
\input{Sec_Formulation}

\section{SANDBOX REINFORCEMENT LEARNING}
\label{sec:Sandbox_RL}
\input{Sec_RL}
\section{FINITE TIME BOUNDS FOR SANDBOX LEARNING}
\label{sec:Sandbox_RL_Analysis}
\input{Sec_RL_analysis}

\section{NUMERICAL RESULTS}
\label{sec:numer}
\input{Sec_numer}

\section{CONCLUSION \& FUTURE WORK}
\label{sec:disc}
\input{Sec_discussion}


\bibliographystyle{plainnat} 
\bibliography{references}
\clearpage

\appendix
\onecolumn
\input{sec_Appendix}

\end{document}

%% file: Sec_Intro.tex
Mean-Field Game (MFG) framework, concurrently introduced by \cite{huang2006large,huang2007large} and \cite{lasry2006jeux,lasry2007mean}, addresses some of the challenges faced by the widely applicable Multi-Agent Reinforcement Learning (MARL) framework~\cite{shoham2007if,ghasemi2020multi,zhang2021multi,mao2022improving}. In particular, MFG framework captures the limiting case where the number of agents $N \rightarrow \infty$ and this deals with the non-stationarity of the environment caused by agents best responding to each other - referred to as the ``curse of many agents" \cite{sonu2017decision}. In the infinite population setting, the effect of individual deviation becomes negligible causing any strategic interaction among the agents to disappear. As a result, it becomes sufficient to consider without loss of generality the interaction between a generic agent and the aggregate behavior of other agents (the mean-field). The solution concept used in MFGs (analog of Nash equilibrium) is called the Mean-Field Equilibrium (MFE). The MFE prescribes a set of control policies which are known to be $\epsilon$-Nash for a large class of $N$-agent games \cite{saldi2018markov}, such that $\epsilon \rightarrow 0$ as $N \rightarrow \infty$. Hence finding the MFE presents a viable method to solving large population games. In this work we propose an RL algorithm to approximate the (stationary) MFE \cite{guo2019learning,xie2021learning} without assuming access to a mean-field oracle (henceforth referred to as oracle). 

Most literature in RL for MFGs assumes access to such an oracle, which is capable of simulating the aggregate behavior of a large number of agents under a given control policy. But this assumption may be prohibitive and the generic agent may not have access to such an oracle, but only knows its own state, action and reward sequence. Hence the question arises:
\begin{center}
\noindent \emph{Can the generic agent provably learn the stationary MFE without access to a mean-field oracle?}
\end{center}
We answer this question in the affirmative by proposing an RL algorithm which computes the MFE without access to an oracle, but instead using the single sample path of the agent (without re-initializations) to approximate the aggregate behavior of large number of agents. We also provide high confidence finite sample bounds for approximation of the MFE to an arbitrary degree. We term this learning approach \emph{Sandbox Learning}, since it allows an agent to approximate equilibrium policies in a multi-agent non-cooperative environment, without interacting with other agents or an oracle. As a result, sandbox learning can be used to provide a \emph{warm-start} to agents before entering an $N$-agent non-cooperative learning environment. 
\subsection{Main Results}
Our core technical insight is that, instead of assuming  access to the oracle, the problem may be cast as a stochastic fixed point problem using the generic agent's single sample path, thus allowing development of \emph{oracle-free} RL algorithm for the MFG. In contrast, prior works require access to mean-field oracle \cite{guo2019learning,xie2021learning,anahtarci2019fitted,fu2019actor},which is a strong assumption, as it implicitly assumes the knowledge of the distribution of all other agents, which never holds in practice.
The main results of the paper are as follows.

\begin{enumerate}[leftmargin=0.44cm]
    \item 
    To efficiently learn the MFE and avoid degenerate policies, the Sandbox learning algorithm simultaneously updates the mean-field and the policy of the agent. This simultaneous update induces a time-varying Markov Chain (MC) for the generic agent which complicates the analysis of the algorithm. In Section \ref{sec:Sandbox_RL}, we craft episodic learning rates for the sole purpose of making the MC \emph{slowly} time-varying inside the episode, making the algorithm amenable for analysis.
    \item In Section \ref{sec:Sandbox_RL_Analysis}, we provide finite sample analysis of Q-learning and dynamics matrix estimation under the slowly time-varying MC setting, using a communicating MDP condition from literature \cite{arslan2016decentralized}. This condition generalizes the pre-existing conditions for RL-MFGs in literature. 
    The slowly time-varying MC setting is shown to introduce a small \emph{drift} in the approximation error, which can be reduced by slowing the inter-episodic learning. Lemmas \ref{lem:tran_prob_conv} and \ref{lem:Q_learn_conv} might be of independent interest to researchers in RL for time-varying MDPs.
    \item The estimates of $Q$-function and dynamics matrix are used to construct approximate optimality and consistency operators, respectively. These operators are used to update the policy and mean-field using two time-scale learning. Finally in Section \ref{sec:Sandbox_RL_Analysis} Corollary \ref{cor:final_bound}, we obtain finite sample convergence bounds of this two time-scale algorithm to an $\epsilon$-neighborhood of stationary MFE, under a standard contraction mapping assumption. 
    \item In Section \ref{sec:numer}, we numerically illustrate the effectiveness of the Sandbox learning algorithm on a congestion game. We empirically demonstrate that the Sandbox learning algorithm performs well even in the absence of the communicating MDP assumption, if there is a single closed communicating class. This is due to the fact that the MC transitions to the communicating class in finite time. 
    \end{enumerate}
    Proofs of theoretical claims are provided in the Appendix. 
\subsection{Relevant Literature}
\label{subsec:rel_lit}
The work most closely related to this paper is \textbf{\cite{angiuli2022unified}} which uses a unified-RL algorithm to solve the MFG problem in cooperative and non-cooperative settings, but lacks rigorous analysis of the RL algorithm. The key differences are that (a) the algorithm in \cite{angiuli2022unified} relies on re-initializations while our algorithm operates on a single sample path, (b) the algorithm proposed in \cite{angiuli2022unified} updates the $Q$-function at a faster time-scale while ours updates the control policy at a faster time-scale, and (c) we explicitly define the learning rates to have a certain episodic structure. These differences are shown to be pivotal in obtaining the finite sample convergence bounds for the Sandbox learning algorithm. 

\textbf{Recent Work of \cite{yardim2022policy}} also deals with RL for MFGs in an oracle-free setting, where $N$ agents independently running policy mirror ascent in a multi-loop algorithm are shown to provably approximate the MFE. One significant difference is that \cite{yardim2022policy} uses $N$-sample paths, compared to the single sample path of our work, to obtain the MFE in $\tilde\Os(\epsilon^{-2})$ time-steps albeit with a bias of $\tilde\Os(1/\sqrt{N})$. Thus accounting for the $N$ sample paths the \emph{sample complexity} of their algorithm is $\tilde\Os(\epsilon^{-4})$, which is similar to the sample complexity of our work. But due to the usage of $N$-sample paths \cite{yardim2022policy} has a better \emph{time complexity} than our work. In addition to the standard contraction mapping assumption in MFG literature (\cite{yardim2022policy} Proposition 2 \& Theorem 2), they also assume contraction (\cite{yardim2022policy} Assumption 2) of the mean-field update under \emph{any} given policy $\pi$. This is assumption requires ergodicity of the generic agent's Markov chain under \emph{any} policy $\pi$, which is stronger than the communicating MDP assumption of our work (Assumption \ref{asm:comm}). This ergodicity condition allows the empirical mean-field of the $N$ agents to \emph{mix} (given a sufficiently long time), and approximate the true mean-field. The ergodicity assumption may not hold in scenarios such as congestion games, as some policies (e.g. always move to neighboring state) may lead to periodicity in the Markov chain. Since our work does not assume ergodicity, the Sandbox learning algorithm approximates the mean-field by instead relying on good estimation of the transition probability matrix (Section \ref{subsec:approx_op}) of the agent under the communicating MDP assumption. Finally, the persistent excitation assumption (\cite{yardim2022policy} Assumption 3) imposes two restrictions on the policy class: (i) the support of the class of policies is non-zero everywhere, and (ii) the policy class is (Bellman) complete. We explicitly include randomness to impose the first restriction, however, we do not require the stronger completeness assumption to obtain the results in our paper. Below we provide a table juxtaposing our work with the contributions of other works in RL for MFGs. 

\begin{center}
	\begin{tabular}{ c c c c c  }
		\hline
   & Oracle-less? & Single sample path & Finite sample bounds \\
		\hline
   \cite{elie2019approximate} & \tikzxmark & \tikzxmark & \tikzxmark \\
   \cite{cui2021approximately} & \tikzxmark & \tikzxmark & \tikzxmark \\  
   \cite{fu2019actor} & \tikzxmark & \tikzxmark & \checkmark \\ 
   \cite{guo2019learning} & \tikzxmark & \tikzxmark & \checkmark \\
   \cite{anahtarci2022q} & \tikzxmark & \tikzxmark & \checkmark \\
   \cite{xie2021learning} & \tikzxmark & \tikzxmark & \checkmark \\
   \cite{angiuli2022unified} & \checkmark & \tikzxmark & \tikzxmark\\
\cite{yardim2022policy} & \checkmark & \tikzxmark & \checkmark \\   
		 This work  & \checkmark & \checkmark & \checkmark\\
		\hline
	\end{tabular}
\end{center}
A complete literature review is provided in the Section \ref{sec:lit_rev}. 

%% file: Sec_Formulation.tex
%

Consider an infinite horizon $N$-agent game over finite state and action spaces $\Ss$ and $\As$, respectively. The state and action of agent $i \in [N]$ at time $t$ are denoted by $s^i_t \in \Ss$ and $a^i_t \in \As$, respectively. Agent $i$'s initial state is drawn from a distribution $s^i_1 \sim p_1 \in \Ps(\Ss)$, and the state dynamics of the agent is coupled with the other agents through the empirical distribution  $e^N_t := \frac{1}{N} \sum_{j \in [N]} \bbone \{s^j_t = s\}$, where we also include agent $i$, without any loss of generality. Agent $i$ generates its actions using policy $\pi^i_t \in \Pi^i_t :=  \{ \pi^i_t \mid \pi^i_t : \Ss \times \Ps(\Ss) \rightarrow \Ps(\As) \}$, dependent on its state and the empirical distribution $e^N_t$. The state of agent $i$ transitions according to
\begin{align}\label{eq:transition_dynamics}
	s^i_{t+1} \sim P(\cdot \mid s^i_t, a^i_t, e^N_t), s^i_1 \sim p_1, a^i_t \sim \pi^i_t(s^i_t,e^N_t) .
\end{align}
Similarly, the reward accrued to the agent depends on its state, action, and the empirical distribution at time $t$, $r^i_t = R(s^i_t,a^i_t,e^N_t) \in [0,1]$. The presence of $e_t^N$ in both \eqref{eq:transition_dynamics} and  $r^i_t$ is a key point of departure from a standard MDP setting, as it permits other agents' possibly non-cooperative behavior to determine the evolution of the state and the reward of agent $i$. The over-arching goal of each agent $i=1,\dots,N$ is to maximize its total reward discounted by a factor $0 < \rho < 1 $, defined as 
\begin{align}\label{eq:value}
	V^i (\pi^i,\pi^{-i}) = \EE \Big[ \sum_{t=1}^{\infty} \rho^t R(s^i_t,a^i_t,e^N_t) \mid s^i_t \sim p_1 \Big],
\end{align}
where $\pi^i := (\pi^i_1,\pi^i_2,\dots) \in \Pi^i$ is the policy of agent $i$ and $\pi^{-i} := \{ \pi^j \}_{j \in [N] \setminus i}$ is the concatenation of policies of all other agents. 
In an $N$-agent non-cooperative game, the dominant solution concept is a Nash equilibrium, where none of the agents can increase their total reward by unilaterally deviating from its Nash policy. 
Based upon this notion, we define an $\epsilon$-Nash equilibrium as follows.
\begin{definition}[\cite{bacsar1998dynamic}] \label{def:eps_Nash}
	A set of policies $\pi^* = \{\pi^{1*}, \ldots, \pi^{N*}\}$ is termed an $\epsilon$-Nash equilibrium if $\forall i \in [N]$,
	%
		$V^i(\pi^{i*},\pi^{-i*}) + \epsilon > V^i(\pi^i,\pi^{-i*}), \forall \pi^i \in \Pi^i. $
	%
%
\end{definition} 
%
If $\epsilon \rightarrow 0$, $\epsilon$-Nash approaches Nash equilibirum. Due to the exponential dependence on the number of agents $N$ required to compute exact Nash equilibria \cite{bacsar1998dynamic}, we restrict focus to computing $\epsilon$-Nash equilibria. In the case that the number of agents $N \rightarrow \infty$, known as the mean-field equilibrium (MFE), one obtains an $\epsilon$-Nash equilibrium \cite{saldi2018markov,moon2014discrete}, specifically, $\epsilon \rightarrow 0$ as $N \rightarrow \infty$.

Therefore, subsequently, we focus on the MFG, the infinite population analog of the $N$-agent game.
The empirical distribution is replaced in that case by a mean field distribution $\mu = \lim_{N,t \rightarrow \infty} e^N_t$, its infinite population stationary counterpart. The stationary MFE of the MFG is guaranteed to exist under certain Lipschitzness assumptions \cite{saldi2018markov,jovanovic1988anonymous} (Assumption \ref{asm:contrct}).  As in the $N$-agent game, the generic agent in a MFG has state space $\Ss$, action space $\As$, and the initial distribution of its state is $p_1 \sim \Ps(\Ss)$. Next, we define the agent's transition dynamics \eqref{eq:transition_dynamics} and total reward \eqref{eq:value} in the mean-field setting with mean-field $\mu \in \Ps(\Ss)$:
\begin{align}
	s_{t+1} \sim P(\cdot \mid s_t, a_t, \mu), s_1 \sim p_1, a_t \sim \pi(s_t,\mu) .
\end{align}
The actions of the generic agent are generated using a stationary stochastic policy $\pi \in \Pi :=\{ \pi: \Ss \times \Ps(\Ss) \rightarrow \Ps(\As)\}$. We restrict ourselves to the set of stationary policies, without loss of generality, since the optimal control policy for an MDP induced by stationary $\mu$ is also stationary \cite{puterman2014markov}. The instantaneous reward $r_t$ accrued to a generic agent at time $t$ is dependent on its state, control action, and the mean-field, that is, $r_t = R(s_t,a_t,\mu)$. The generic agent aims to maximize its total discounted reward given the mean-field $\mu$ and with the discount factor $0 < \rho < 1$,
\begin{align}\label{eq:value_mfg}
	V_{\pi,\mu} := \EE \Big[ \sum_{t=1}^{\infty} \rho^t R(s_t,a_t,\mu) \mid s_1 \sim p_1 \Big]. 
\end{align}
Next we define the Mean-Field Equilibrium (MFE) by introducing two operators. First define the \emph{optimality} operator $\Gamma_1(\mu) := \argmax_{\pi} V_{\pi,\mu}$ as the operator which outputs the optimal policy for the MDP induced by mean-field $\mu$. We consider policies where the probability is split evenly among optimal actions for a given state and mean-field. We also define $\Gamma_2(\pi,\mu)$ as the \emph{consistency} operator which computes mean-field consistent with the policy $\pi$ and mean-field $\mu$. If $\mu' = \Gamma_2(\pi,\mu)$, then $\forall s' \in \Ss$
 \begin{align} \label{eq:gamma_2}
 	\mu'(s') = \sum_{(s,a) \in \Ss \times \As} P(s' \mid s, a, \mu) \pi(a \mid s, \mu) \mu(s).
 \end{align}
This is also referred to as the Fokker-Planck-Kolmogorov equation in the literature \cite{bensoussan2015master}, and versions of it appear in the literature on probability flow equations in MDPs \cite{puterman2014markov}. 
Consistency means that if infinitely many agents (with initial distribution $\mu$) follow a control policy $\pi$, the resulting distribution will be $\mu'$. 
Using these two operators, we can define the MFE of the MFG as follows.
\begin{definition}[{\cite{saldi2018markov}}] \label{def:MFE}
	The pair $(\tilde\pi,\tilde\mu)$ is an MFE of the MFG if $\tilde\pi = \Gamma_1(\tilde\mu)$ and $\tilde\mu = \Gamma_2(\tilde\pi,\tilde\mu)$.
%
\end{definition}
Intuitively this two-part coupled definition can be interpreted as (1) $\tilde\pi$ is the optimal policy for the MDP induced by mean-field $\tilde\mu$, and (2) mean-field $\tilde\mu$ is consistent with the control policy $\tilde\pi$. 
A naive way of approximating the MFE could be through repeated use of the composite operator $\Gamma_2(\Gamma_1(\cdot), \cdot)$ but this iteration is known to be non-contractive (\cite{cui2021approximately}). Instead we replace $\Gamma_1(\cdot)$ with the \emph{approximate optimality} operator $\Gamma^\lambda_1 (\mu):= \softmax{\lambda}{\cdot,Q^*_{\mu}}$, where $Q^*_{\mu}$ is the $Q$-function of the MDP induced by mean-field $\mu$ and, the $\softmax{\lambda}{\cdot}$ function is defined as
\begin{align} \label{eq:softmax}
	\softmax{\lambda}{s,Q}_a := \frac{\exp(\lambda Q(s,a))}{\sum_{a' \in \As} \exp(\lambda Q(s,a'))},
\end{align}
$\forall s \in \Ss, \forall a \in \As$. Evidently as $\lambda \rightarrow \infty$, $\Gamma^\lambda_1 \rightarrow \Gamma_1$. Next using the approximate optimality operator $\Gamma^\lambda_1$ we define an approximate MFE known as \emph{Boltzman}-MFE (B-MFE).
\begin{definition}[{\cite{cui2021approximately}}] \label{def:B_MFE}
For a given $\lambda >0$, the pair $(\pi^*,\mu^*)$ is a Boltzman-MFE (B-MFE) of the MFG if $\pi^* = \Gamma^\lambda_1(\mu^*)$ and $\mu^* = \Gamma_2(\pi^*,\mu^*)$.
\end{definition}
The Boltzman-MFE is an approximate MFE and approaches the MFE as $\lambda \rightarrow \infty$ (Theorem 4, \cite{cui2021approximately}). Henceforth, we will devote ourselves to finding the B-MFE for a large enough $\lambda$, so as to closely approximate the MFE. Next we introduce the standard contraction mapping assumption in MFGs \cite{guo2019learning,xie2021learning}.
\begin{assumption} \label{asm:contrct}
	There exists a $\lambda >0$ and Lipschitz constants $d_1,d_2$ and $d_3$ such  that 
			\begin{align*}
			\lVert \Gamma^\lambda_1(\mu) - \Gamma^\lambda_1(\mu') \rVert_{TV} & \leq d_1 \lVert \mu - \mu' \rVert_1, \\
			\lVert \Gamma_2(\pi,\mu) - \Gamma_2(\pi',\mu) \rVert_1 & \leq d_2 \lVert \pi - \pi' \rVert_{TV}, \\
            \lVert \Gamma_2(\pi,\mu) - \Gamma_2(\pi,\mu') \rVert_1 & \leq d_3 \lVert \mu - \mu' \rVert_1
			\end{align*}
		and $d:=d_1 d_2 + d_3 < 1$ for policies $\pi,\pi' \in \Pi$ and mean-fields $\mu,\mu' \in \Ps(\Ss)$.
\end{assumption}
Assumption \ref{asm:contrct} is guaranteed to be true for a small enough $\lambda >0$ \cite{cui2021approximately}. This results in a 
trade-off as higher values of $\lambda$ increase \emph{closeness} between MFE and B-MFE, but may cause Assumption \ref{asm:contrct} to be violated, and vice-versa. This issue is well-known in MFGs over finite state and action spaces \cite{cui2021approximately}. Contraction mapping conditions (as in Assumption 1) are widely used in RL for standard MFGs \cite{guo2019learning,xie2021learning,fu2019actor}. 
Lemma 5 in \cite{guo2019learning} provides candidate values for these constants. The $\lVert \cdot \lVert_{TV}$ norm used in Assumption \ref{asm:contrct} is the Total variation bound \cite{cui2021approximately} and is defined for a function  $f : \As \times \Ss \rightarrow \RR$ such that $\lVert f \rVert_{TV} := \max_{s \in \Ss} \sum_{a \in \As} \lvert f(a \mid s) \rvert$. 
Under Assumption \ref{asm:contrct}, the existence and uniqueness of the B-MFE of the MFG has been proven in literature \cite{cui2021approximately,guo2019learning,xie2021learning} using the standard contraction mapping theorem. 
Hence, the B-MFE approximates the MFE for large values of $\lambda$ and the MFE is known to be $\epsilon$-Nash for the finite population game (Theorem 2.3 \cite{saldi2018markov}). 
In the next section, we propose an RL algorithm to approximate the B-MFE without access to a mean-field oracle, by utilizing the sample path of a generic agent itself. 

%% file: Sec_RL.tex
%


Consider a setting where a generic agent has no knowledge of the transition probability $P$, the functional form of the reward $R$ or a mean-field oracle, which is often required in such studies -- see \cite{guo2019learning,fu2019actor,xie2021learning,cui2021approximately}. 
In this section, we propose a Sandbox RL algorithm to compute the B-MFE. Our methodology operates by updating the mean-field and the control policy concurrently using approximations of the optimality and consistency operators, $\Gamma^\lambda_1$  and $\Gamma_2$, respectively, defined prior to Definition \ref{def:B_MFE}. The approximation to $\Gamma^\lambda_1$ is defined by $\softmax{\lambda}{\cdot}$ of estimated $Q$-function obtained using $Q$-learning update, whereas approximation of operator $\Gamma_2$ relies on estimating the transition probabilities of the Markov Chain (MC) of the generic agent. But the concurrent update of mean-field and control policy causes the MC of the generic agent to be time-varying. This time-varying MC setting may cause instability in the approximation of the operators, resulting in divergence of mean-field and control policy updates.
\begin{figure}[h!]
    \centering
	\includegraphics[width=0.6\textwidth]{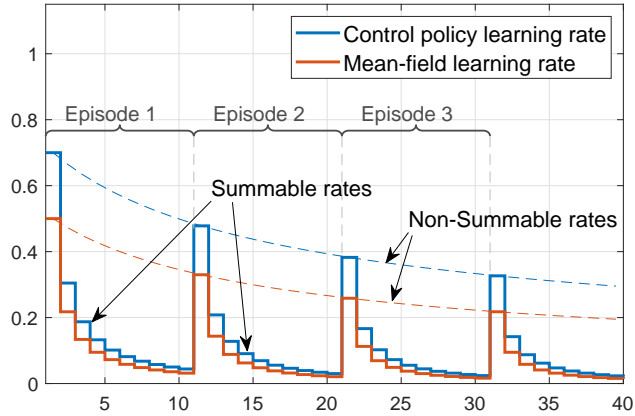}
	\caption{Episodic Two time-scale learning rate}
	\label{fig:learn_rates}
\end{figure}

To ensure good approximation of operators, we adopt an episodic two time-scale learning rate as shown in Figure \ref{fig:learn_rates}. Inside an episode, the learning rates are summable (or fast-decaying), allowing the degree of non-stationarity in the MC inside the episode to be \emph{slowly time-varying}. Doing so then enables us to ensure that the approximation errors of the optimality and consistency operators are under control. Therefore, given a reasonable estimate for the consistency operator, the control policy is updated on a faster time-scale. Similarly the mean-field is updated at a slower time-scale using the consistency operator. We note that inverting the entity updated on a faster/slower time-scale will result in the solution to the Mean-Field Control problem \cite{angiuli2022unified}. 
In the following subsection we describe how we can estimate the two operators. 

\subsection{Approximate Mean-Field consistency and optimality operators} \label{subsec:approx_op}
We start by describing how the Sandbox learning algorithm uses the MC of the generic agent to approximate the consistency operator $\Gamma_2$. Recalling the definition of $\Gamma_2$ \eqref{eq:gamma_2}, if $\mu' = \Gamma_2(\pi,\mu)$, then
\begin{align*}
    \mu'(s') & = \sum_{s \in \Ss} \sum_{a \in \As} P(s' \mid s,a,\mu) \pi(a \mid s,\mu) \mu(s)  = \sum_{s \in \Ss} P_{\pi,\mu} (s,s') \mu(s), \hspace{0.2cm} \forall s' \in \Ss
\end{align*}
where $P_{\pi,\mu}$ is the transition dynamics matrix of the generic agent under control law $\pi$ and mean-field $\mu$. Hence if $\mu' = \Gamma_2(\pi,\mu)$, the vector $\mu' \in \Ps(\Ss)$ can be written as
\begin{align}\label{eq:mu_fixed_point}
	\mu' & = P_{\pi,\mu}^\top \mu
\end{align}
To come up with an estimator for $\Gamma_2$ we will need to estimate the dynamics matrix $P_{\pi,\mu}$. Toward this end, we can take a sample path of the Markov chain induced by $\pi$ and $\mu$ of length $T$ to obtain approximation of $\mu'$ through the use of an estimation of the occupancy (visitation) measure, and we can determine to what extent this estimate would be optimal through its ability to solve equation \eqref{eq:mu_fixed_point}. More specifically, for a fixed pair of states $ (i,j) \in \Ss \times \Ss$, the empirical transition probabilities $\hat P$ can be computed by keeping track of the state visitation numbers $N(i)$ and $N({i,j})$ as follows:
\begin{align}\label{eq:consistency_estimator}
	\hat P(i,j) & = \frac{N(i,j) + 1/S}{N(i) + 1},
\end{align}
where $ N(i,j) = \big\lvert \{l \in [T] : s_l = i, s_{l+1} = j\} \big\rvert, N(i) = \sum_{j \in \Ss} N^k_{t}(i,j)$ and $s_t$ is the state visited by the MC at time $t \in [T]$. Notice that we use smoothing (by adding $1/S$ and $1$ to the numerator and the denominator, respectively) to avoid degenerate cases during the transition probability estimation. 
The transition probabilities $\hat P$ approximate the true transition probabilities $P_{\pi,\mu}$. Hence the approximate consistency operator is then given by $\hat{P}^\top \mu$, and the associated mean-field is updated by sequentially applying  $\hat{P}^\top \mu$ with a specific step-size [cf. \eqref{eq:step-sizes}] in \eqref{eq:algorithm_iteration1}, which we defer to the next subsection in order to underscore its concurrence with policy updates that are derived in terms of the Bellman equations. 

Now we describe how the Sandbox learning algorithm approximates the optimality operator $\Gamma^\lambda_1$. As described in Section \ref{sec:form} $\Gamma^\lambda_1 := \softmax{\lambda}{\cdot,Q^*_{\mu}}$, where  $\softmax{\lambda}{\cdot}$ is defined in \eqref{eq:softmax} 
and $Q^*_\mu(s,a) := \argmax_\pi \EE[ \sum_{t=1}^\infty R(s_t,a_t,\mu )| s_1=s, a_1=a ]$ is the optimal $Q$-function for the MDP induced by the mean-field $\mu$ and is the fixed point of the Bellman equation $$Q^*_\mu(s,a) = R(s,a,\mu) + \rho \EE_{s' \sim P(\cdot)}[\max_{a'} Q^*_\mu(s',a') ].$$
The algorithm uses $Q$-learning update to approximate the optimal $Q$-function. The asynchronous Q-learning update \cite{lewis2012optimal,even2003learning} can be written as follows,
\begin{align}
	& Q_{t+1}(s_t,a_t) = (1 - \beta_t) Q_t (s_t) + \beta_t \big( r_t + \rho \max_{a \in \As} Q_t(s_{t+1},a) \big), \label{eq:Q_learn}
\end{align}
 where $\beta_t := c_\beta/(t+1)^\nu$ and $0.5 < \nu \leq 1$. 
Let us denote the approximate optimality operator as $\hGamma_1 := \softmax{\lambda}{\cdot,\hat{Q}}$, where $\hat{Q}$ is the estimate obtained using the $Q$-learning update \eqref{eq:Q_learn}. The estimation error in $\hat \Gamma_1$ is due to the estimation error in the $Q$-learning, and is monotonically increasing with $\lambda$. 
With this technical machinery introduced for the approximate consistency operator and the Bellman operator, 
we are now ready to introduce the Sandbox learning algorithm. This is the focus of the following subsection.

\subsection{Sandbox Reinforcement Learning algorithm}\label{subsec:sandbox_RL}
%

The Sandbox learning algorithm is presented in Algorithm \ref{alg:Single_loop_RL}. Throughout the algorithm the superscript $k \in [K]$ refers to the episode, and subscript $t \in [T]$ refers to the timestep inside the episode. Each episode $k$ lasts for $T$ timesteps. The state $s^1_1$ is initialized using distribution $p_1$ and a new state $s^k_{t+1}$ is generated at each timestep $t$ (line \ref{algline:state_gen}), and hence the algorithm evolves over a single sample path (of the generic agent) without re-initialization. The mean-field $\mu^k_t$ and the control policy $\pi^k_t$ are updated at each timestep according to 
\begin{align}
	\mu^k_t & = \PP_{S(\epsilon^{\net})} \big[(1-c^k_{\mu,t}) \mu^k_{t-1} + c^k_{\mu,t} \big(  (\hat{P}^k_t)^\top \mu^k_{t-1} \big), \bbone_{\{ t = 1 \}} \big], \label{eq:algorithm_iteration1} \\
	\pi^k_{t} & = (1-c^k_{\pi,t}) \pi^k_{t-1} +  c^k_{\pi,t} \big((1-\psi^k_t)\softmax{\lambda}{\cdot,Q^k_t} + \psi^k_t \mathds{1}_{|\As|}  \big). \label{eq:algorithm_iteration2} 
\end{align}
\noindent 
The update of mean-field involves the operation $\PP_{S(\epsilon^\net)}[\mu,x]$, which projects $\mu$ onto the $\epsilon$-net $S(\epsilon^\net)$ 
if and only if $x = 1$. This projection step is performed on the first time-step of each episode $k$. In the analysis (Section \ref{sec:Sandbox_RL_Analysis}) we show that $\epsilon^\net = \Os(\epsilon^2)$ at worst, where $\epsilon > 0$ is the approximation error in the B-MFE. The update of the mean-field is performed using the approximate consistency operator $(\hat{P}^k_t)^\top \mu^k_{t-1}$, and the control policy is updated using the approximate optimality operator $\softmax{\lambda}{\cdot,Q^k_t}$. 
The control policy updates also involve an exploration noise $\psi^k_t \mathds{1}_{|\As|}$, which results in sufficient exploration of the state-action space (Lemma \ref{lem:suff_exp}) without effecting the convergence bounds (Theorem \ref{thm:conv_bound} \& Corollary \ref{cor:final_bound}). The expression for exploration coefficient $\psi^k_t$ is provided in the proof of Lemma \ref{lem:suff_exp}. 
The learning rates for the update, $c^k_{\mu,t}$ and $c^k_{\pi,t}$, are episodic two time-scale:
\begin{align}\label{eq:step-sizes}
	c^k_{\mu,t} = \frac{c_\mu}{k^\gamma}\frac{1}{t^\zeta}, \hspace{0.2cm} c^k_{\pi,t} = \frac{c_\pi}{k^\theta}\frac{1}{t^\zeta},
\end{align}
where $0 < \theta < \gamma < 1 < \zeta < \infty$. The episodic nature of the learning rates is due to the $1/t^\zeta$ factor, $\zeta > 1$ in both rates, which makes it summable, resulting in slowly time varying MC inside the episode. The two time-scale nature of the learning rate is due to $\theta < \gamma$ where the update of the policy $\pi^k_t$ is faster than that of the mean-field $\mu^k_t$. Furthermore, the learning rates $c^k_{\mu,t}$ and $c^k_{\pi,t}$ are non-summable since $0 < \theta,\gamma<1$. This episodic two time-scale nature is pivotal in proving that Sandbox RL converges to the B-MFE of the MFG as shown in the next section.
\begin{algorithm}[h!]
	\caption{Sandbox RL for MFG}
	\begin{algorithmic}[1] \label{alg:Single_loop_RL}
		\STATE {\bf Initialize}: initial state $s^1_1 \sim p_1$, policy $\pi^1_0$ and mean-field $\mu^1_0$
		\FOR {$k \in \{1,2,\ldots,K\}$}
		\FOR {$t \in \{1,2,\ldots,T\}$}
            \STATE Update $\mu^k_t, \pi^k_{t}$ using \eqref{eq:algorithm_iteration1}, \eqref{eq:algorithm_iteration2} respectively. \label{algline:cp_mf_update}
		\STATE Generate single transition $s^k_{t+1} \sim P(\cdot \mid s^k_t,a^k_t,\mu^k_t)$ and reward $r^k_t = R(s^k_t,a^k_t,\mu^k_t)$ with $a^k_t \sim \pi^k_t(s^k_t,\mu^k_{t})$. \label{algline:state_gen}
		\STATE {\bf Transition probability estimation:} For $(i,j) \in \Ss \times \Ss$ \label{algline:tran_prob_est}
		\begin{align}
			\hat P^k_{t+1}(i,j) & = \frac{N^k_{t}(i,j) + 1/S}{N^k_{t}(i) + 1}, \label{eq:P_est}
		\end{align}
		where $N^k_{t}(i,j) = \big\lvert \{l \in [t] : s^k_l = i, s^k_{l+1} = j\} \big\rvert, N^k_{t}(i) = \sum_{j \in \Ss} N^k_{t}(i,j)$.
		\STATE {\bf Q-learning: \label{algline:q_learn}
		$Q^k_{t+1}(s^k_t,a^k_t) = (1 - \beta_t) Q^k_t + \beta_t \big( r^k_t + \rho \max_{a \in \As} Q^k_t(s^k_{t+1},a) \big)$}
		\ENDFOR
		\STATE {$\hat{P}^{k+1}_1 = \hat{P}^k_{T+1}, Q^{k+1}_1 = Q^k_{T+1}, \mu^{k+1}_0 = \mu^k_T, \pi^{k+1}_0 = \pi^k_T, s^{k+1}_1 = s^k_{T+1}$}
		\ENDFOR
		\STATE {\bf Output:} Approximate B-MFE $(\frac{1}{K} \sum_{k=1}^{K-1} \pi^k_1, \frac{1}{K} \sum_{k=1}^{K-1} \mu^k_1)$.
	\end{algorithmic}
\end{algorithm}

%% file: Sec_RL_analysis.tex

Most results in RL for MFGs  break down in our setting as they assume a time invariant MC. In contrast, concurrent update of the mean-field and the control policy in the Sandbox learning algorithm induces a time-varying MC. In this section we analyze how the \emph{slowly} time-varying MC under the episodic learning rates \eqref{eq:algorithm_iteration1}-\eqref{eq:algorithm_iteration2} is more amenable to analysis and leads to good approximations of $\Gamma^\lambda_1$ and $\Gamma_2$ operators. Toward this end we first prove convergence of the transition probability and $Q$-learning estimation inside each episode $k \in [K]$ in Lemmas \ref{lem:tran_prob_conv} and \ref{lem:Q_learn_conv}. These results are worthy of interest independent of the Sandbox learning algorithm, due to the slowly time-varying MC setting. 
In contrast, earlier works \cite{guo2019learning,xie2021learning,anahtarci2019fitted} deal with approximating just $\Gamma^\lambda_1$ under a time invariant MC. Then in Theorem \ref{thm:conv_bound} we show that good approximation of $\Gamma^\lambda_1$ and $\Gamma_2$ operators (due to good $Q$-learning and transition probability estimation, respectively) results in decreasing average error in policy and mean-field. Finally, in Corollary \ref{cor:final_bound}, we present finite sample analysis for the two time-scale Sandbox learning algorithm. 

Lemma \ref{lem:tran_prob_conv} presents error bounds on transition probability estimation \eqref{eq:consistency_estimator} for a slowly time-varying MC, under a communicating MDP condition as given below. 
Assumption \ref{asm:comm} \emph{generalizes} the pre-existing conditions for RL-MFGs in literature. The online RL-MFG works of \cite{guo2019learning} and \cite{xie2021learning} (and references therein \cite{shah2018q,farahmand2016regularized}) 
assume either a covering time assumption or require i.i.d. samples from stationary distribution. Similarly \cite{yardim2022policy} requires ergodicity of the generic agent's MC under any policy. 
The offline RL-MFG works of \cite{anahtarci2019fitted} and \cite{fu2019actor} rely on i.i.d. samples from unique stationary distribution of MC which requires ergodicity. 
Communicating MDP (Assumption \ref{asm:comm}) is more general than covering time or ergodicity conditions \cite{chandrasekaran2021learning}. Before stating the communicating MDP condition, we introduce the set $S(\epsilon^\net)$ which is a set of mean-field distributions. This set (also termed $\epsilon$-net \cite{guo2019learning} over $\Ps(\Ss)$) defined as $S(\epsilon^\net) = \{\mu^1, \ldots, \mu^{N_\net} \} \subset \Ps(\Ss)$ is a finite set of simplexes over $\Ss$ such that $\lVert \mu - \mu' \rVert_1 \leq \epsilon^\net$ for any $ \mu \in \Ps(\Ss), \exists \mu' \in S(\epsilon^\net)$. The existence of the set is guaranteed due to the compactness of $\Ps(\Ss)$.
\begin{assumption} \label{asm:comm} {\bf (Communicating MDPs 
\cite{arslan2016decentralized})}
    For any mean-field $\tilde\mu \in S(\epsilon^\net)$ (which is a finite set) and any pair of states $s,s' \in \Ss$, there exists a finite horizon $H(\tilde\mu)$ such that for $t \geq H(\tilde\mu)$ there exists a set of actions $\tilde{a}_1,\ldots,\tilde{a}_{t}$,
    \begin{align*}
        & P\big(s_{t} = s' \mid a_1=\tilde{a}_1,\ldots,a_{t}=\tilde{a}_{t}, s_1 = s, \mu = \tilde\mu\big) > 0.
    \end{align*}
\end{assumption} 
Informally Assumption \ref{asm:comm} means that every agent in the game has a path from any state to any other state for mean-fields in $S(\epsilon^\net)$. Assumption \ref{asm:comm} is satisfied in several real-world scenarios. Production of an exhaustible resource by competing producers (e.g. oil) is a typical multi-agent setting where Assumptionm \ref{asm:comm} is satisfied \cite{Guéant2011}, since the agents can achieve any level of reserve by increasing/decreasing their production. Capital accumulation games \cite{fershtman1984capital} and asset management games \cite{lacker2019mean} have a similar structure thus satisfying Assumptionm \ref{asm:comm}. It is also satisfied in cyber-security applications \cite{kolokoltsov2016mean}, as any infection state can be reached by choosing the right policy and a strictly positive MF $\epsilon$-net. 
In Section \ref{sec:numer} we numerically investigate a setting where such an assumption is not satisfied. Next, under the communicating MDP assumption, we prove sufficient exploration of state and action space, under the policy update \eqref{eq:algorithm_iteration2}. 
\begin{lemma} \label{lem:suff_exp}{\bf (Sufficient Exploration)}
    If Assumption \ref{asm:comm} is satisfied, stochastic kernel $P(\cdot \mid s,a,\mu)$ is Lipschitz in $\mu$, and $\zeta$ is large enough, then under the control policy update (Algorithm \ref{alg:Single_loop_RL} line \ref{algline:cp_mf_update}), there exists a $\sigma \in (0,1)$ such that for any $(s,a) \in \lvert \Ss \times \As \rvert$ and $t \geq H:=\max_{\tilde\mu \in S(\epsilon^\net)}H(\tilde\mu)$, $P((s_t,a_t) = (s,a) \mid \Fs_{t-H}) \geq \sigma$.
\end{lemma}
Lemma \ref{lem:suff_exp} implies that the communicating MDP assumption coupled with the policy update \eqref{eq:algorithm_iteration2}, is more general than the sufficient exploration condition used in Q-learning for MDPs (\cite{qu2020finite} Assumption 3) as well as for $N$-player stochastic games (\cite{hu2003nash} Assumption 1). Furthermore, it is more general than an ergodicity assumption used in the stochastic optimization literature \cite{srikant2019finite}. We further note that the sufficient exploration condition has also been used in two time-scale settings in the literature \cite{wu2020finite}. 
    Next we quantify the error in transition probability estimation in Lemma \ref{lem:tran_prob_conv} under Assumption \ref{asm:comm} and Lipschitz conditions on transition probability $P$ \cite{angiuli2022unified}. The estimation error is denoted by $\epsilon^k_P$, and is the norm of the difference between the transition probability estimate $\hat P^k_T$ and the true transition probability induced by the control policy and the mean-field at the first timestep ($\pi^k_1, \mu^k_1$, respectively). 
\begin{lemma} \label{lem:tran_prob_conv}
	Given that Assumption \ref{asm:comm} is satisfied and transition probability $P_{\pi,\mu}$ is Lipschitz in policy $\pi$ and mean-field $\mu$ such that $\lVert P_{\pi,\mu} - P_{\pi',\mu} \rVert_F \leq L^\pi_P\lVert \pi - \pi' \rVert_{TV}$ and $\lVert P_{\pi,\mu} - P_{\pi,\mu'} \rVert_F \leq L^\mu_P\lVert \mu - \mu' \rVert_1$, the error in transition probability estimation for episode $k$ is
        \begin{align*}
		& \epsilon^k_P := \lVert \hat{P}^k_{T} - P_{\pi^k_1,\mu^k_1} \rVert_F = \tilde\Os(T^{-1/2}) + \tilde\Os(T^{-1}) + \Os(2^{1-\zeta})
	\end{align*}
	with probability at least $1-\delta_P$ where 
 $P_{\pi^k_1,\mu^k_1}$ is the transition probability under control law $\pi^k_1$ and mean-field $\mu^k_1$.
\end{lemma}
\azedit{The Lipschitz conditions in Lemma 1 will follow if transition probability is continuous in the mean-field $\mu$ and the policy $\pi(\cdot \mid s)$, due to the compactness of mean-field and policy spaces, $\Ps(\Ss)$ and $\Ps(\As)$, respectively. And in most real-world examples, such as asset \cite{reis2019forward} and crowd management \cite{priuli2014first}, continuity of transition probability w.r.t. mean-field and policy is ensured.} 
Lemma \ref{lem:tran_prob_conv} shows that the estimation error $\epsilon^k_P$ contains a drift term $\Os(2^{1-\zeta})$ due to the slowly time-varying MC setting which can be decreased by increasing the inter-episodic learning parameter $\zeta$. Aside from drift, $\epsilon^k_P$ grows at $\tilde\Os(T^{-1/2}) + \tilde\Os(T^{-1})$, where $\tilde \Os$ hides logarithmic factors. Hence increasing the duration of episode $T$ will result in decrease in estimation error. The proof of the lemma is given in the Appendix and relies on Freedman's inequality \cite{freedman1975tail}. 

Next we analyze the error in $Q$-learning estimation \eqref{eq:Q_learn} for each episode $k \in [K]$. This update has been shown to converge to the optimal $Q$ function under a sufficient exploration condition (stronger than Assumption \ref{asm:comm}) for a time invariant MC \cite{even2003learning,qu2020finite}. In Lemma \ref{lem:Q_learn_conv} we show that this update converges (albeit with a drift) under the comunicating MDP condition for the slowly time-varying MC setting and with $0.5 < \nu \leq 1$. This estimation error is denoted by $\epsilon^k_Q$, and is the norm of the difference between the estimate of the optimal $Q$-function $Q^k_T$ and the true $Q$-function $Q^{*,k}_1 := Q^*_{\mu^k_1}$ for the MDP induced by the mean-field $\mu^k_1$. As in Lemma \ref{lem:tran_prob_conv}, a drift term $\Os(2^{1-\zeta})$ creeps in due to the slowly time-varying MC setting.
\begin{lemma} \label{lem:Q_learn_conv}
	Under Assumption \ref{asm:comm}, 
	the estimation error in $Q$-learning for episode $k$ is
        \begin{align*}
		& \epsilon^k_Q := \lVert Q^k_{T} - Q^{*,k}_1 \rVert_\infty  = \Os(T^{1-2\nu}) + \Os(T^{1-\zeta-\nu}) + \tilde \Os(T^{1/2-\nu}) + \Os(2^{1-\zeta})
	\end{align*}
	with probability at least $1-\delta_Q$.
	
\end{lemma}
The slowly time-varying MC also contributes $\Os(T^{1-\zeta-\nu})$ error component but that is dominated by the $\Os(T^{1-2\nu})$ term due to $\nu \leq 1 < \zeta$. The error terms, which are $ \Os(T^{1-2\nu})$ and $\tilde \Os(T^{1/2-\nu})$, are decreasing for increasing $T$, and hence to get a small $\epsilon^k_Q$ we need a large enough $T$. 
Combining the bounds from Lemmas \ref{lem:tran_prob_conv} and \ref{lem:Q_learn_conv} we can surmise that given that the inter-episode learning parameter $\zeta$ and the episode length $T$ are large enough, the transition probability estimation and $Q$-learning will be good enough, leading to good approximations of the consistency and optimality operators, $\Gamma_2$ and $\Gamma^\lambda_1$, respectively.

Now we are in a position to present Theorem \ref{thm:conv_bound}, relying on good approximations of the consistency and optimality operators. Theorem \ref{thm:conv_bound} below bounds the average error in policy $e^k_\pi := \lVert \pi^k_1 - \Gamma^\lambda_1(\mu^k_1) \rVert_{TV}$ and mean-field $e^k_\mu := \lVert \mu^k_1 - \mu^* \rVert_1$ over episodes $k \in [K]$, given that $\epsilon^k_P \leq \epsilon_P, \epsilon^k_Q \leq \epsilon_Q / \log(K)$ for some $\epsilon_P, \epsilon_Q > 0$.
\begin{theorem} \label{thm:conv_bound}
	Let the approximation errors be denoted by $e^k_\pi := \lVert \pi^k_1 - \Gamma^\lambda_1(\mu^k_1) \rVert_{TV}$ and $e^k_\mu := \lVert \mu^k_1 - \mu^* \rVert_1$ \azedit{and $\epsilon^\net \leq (c_\mu \bar{d} \epsilon)/K^\gamma$ for $\epsilon > 0$}. Under Assumptions \ref{asm:contrct}-\ref{asm:comm}, with the estimation errors satisfying $\epsilon^k_P \leq \epsilon_P, \epsilon^k_Q  \leq \epsilon_Q / \lambda,$ for some $\epsilon_P, \epsilon_Q > 0$, the average approximation errors decrease at the following rates:
 	\begin{align*}
	\frac{1}{K}\sum_{k=1}^{K-1} e^k_\pi = & \Os(K^{\theta-1})   + \Os(\epsilon_Q) + \Os (K^{\theta-\gamma}) + \Os(K^{-1})  + \Os(2^{1-\zeta}), \\
	\frac{1}{K}\sum_{k=1}^{K-1} e^k_\mu = &\Os(K^{\gamma-1}) \azedit{+ \Os(\epsilon)} + \Os(\epsilon_P) + \frac{1}{K}\sum_{k=1}^{K-1} e^k_\pi
	\end{align*}
with probability at least $1-\delta_Q$, where $ 0 < \theta < \gamma < 1 < \zeta < \infty$.
\end{theorem}
The challenges in establishing Theorem \ref{thm:conv_bound} are due to the two time-scale learning rates and non-regularized MFG setting. The proof of Theorem \ref{thm:conv_bound} keeps track of errors $e^k_\pi$ and $e^k_\mu$ for each episode $k$ and the average of these errors is shown to approach $0$ due to tight approximation of the optimality and consistency operators and the two time-scale update under the contraction Assumption \ref{asm:contrct}. Apart from the familiar drift terms $\Os(2^{1-\zeta})$ and estimation error bounds $\epsilon_P$ and $\epsilon_Q$, all other terms are decreasing with increasing total number of episodes $K$ at rates governed by $\theta, \gamma$ and $\zeta$. 
Next we present a corollary to Theorem \ref{thm:conv_bound} which gives us the final bound quantifying the approximation error between output of the Sandbox learning algorithm and the B-MFE of the MFG. The proof of Corollary \ref{cor:final_bound} depends on the result of Theorem \ref{thm:conv_bound} and is provided in the Appendix.
\begin{corollary} \label{cor:final_bound}
	If all conditions in Theorem \ref{thm:conv_bound} are satisfied, we have
        \begin{align*}
		& \bigg\lVert \frac{1}{K} \sum_{k=1}^{K-1} \pi^k_1 - \pi^* \bigg\rVert_{TV} + \bigg\lVert \frac{1}{K} \sum_{k=1}^{K-1} \mu^k_1 - \mu^* \bigg\rVert_1 \\
            & \hspace{2cm} =  \Os(K^{\gamma-1}) + \Os(2^{1-\zeta}) + \Os(\epsilon_P) + \Os(K^{\theta-1})   + \Os(\epsilon_Q) 
  + \Os (K^{\theta-\gamma}) + \Os(\epsilon). 
	\end{align*}
\end{corollary}
The terms $\Os(K^{\theta -1}), \Os(K^{\gamma-1})$ and $\Os(K^{\theta-\gamma})$ enter due to the two-time-scale learning setting and are monotonically decreasing with the number of episodes $K$. The convergence rates of these terms can be tuned by choosing $\theta$ and $\gamma$ as explained next. 
The terms $\Os(\epsilon_Q)$ and $\Os(\epsilon_P)$ enter into the analysis due to the estimation errors in the $Q$-function and the dynamics matrix $P$, respectively. These quantities can be made arbitrarily small by increasing the number of timesteps in each episode $T$, due to Lemmas 1 and 2. Lastly, the term $\Os(2^{1-\gamma})$ is a drift term which enters due to inter-episodic learning. This can be made arbitrarily small by increasing the value of $\gamma$. But the value of $\gamma \neq \infty$ as that stops inter-episodic learning and may cause degenerate policies. The error introduced by the projection step in the mean-field update line \ref{algline:cp_mf_update} of Algorithm \ref{alg:Single_loop_RL} is $\Os(\epsilon)$.

If the learning rates are chosen such that $\theta = 0.01, \gamma = 0.5, \nu = 1, \zeta = \Omega(\log(1/\epsilon))$, then the output of Algorithm \ref{alg:Single_loop_RL} will be $\epsilon$ close to the B-MFE with high probability (for small enough $\epsilon>0$), given that episode length is $T = \Omega(\epsilon^{-2})$ and the number of episodes is $K = \Omega(\epsilon^{-2})$. Notice that under these conditions, $\epsilon^\net = \Os(\epsilon^2)$. Hence the sample complexity of the algorithm is $\Os(\epsilon^{-4})$. Finally, the difference between the MFE and the B-MFE will be determined by the $\lambda$ parameter with higher values resulting in a close approximation of MFE by the B-MFE \cite{cui2021approximately}. In the next section we apply the algorithm to a congestion game.

%% file: Sec_numer.tex
We simulate Sandbox learning for a congestion MFG \cite{toumi2020tractable} on a grid. In particular we investigate the convergence of the algorithm for the cases (1) when the communicating MDP assumption is satisfied, and (2) when it is not satisfied. The state space $\Ss = \{1,\ldots,5\}^2$, action space $\As = \{-1,1\}^2$, and discount factor $\rho = 0.7$. The initial distribution of the agent $p_1$ is uniform over the state space. 
If the agent takes action $a \in \As$ in state $s \in \Ss$, the resulting state will be $s' = \PP_\Ss[s+a]$ with probability $1-p$ or the agent may be `jostled' into one of the neighboring states, with probability $p$ for small $p > 0$. This stochasticity is meant to model the jostling present in crowds. The agent's reward is $r(s,\mu) = (1-c \cdot \mu(s)) \cdot R(s)$ where $c = 0.5$ is the congestion averse parameter and $R(s)$ is the state dependent component of the reward. The state-dependent rewards $R(s)$ are concentrated around favorable states $\{(3,3),(3,4),(4,3),(4,4)\}$. Hence the agent prefers states with higher $R(s)$ values but might be deterred by the congestion in the state $\mu(s)$.

The initial control policy $\pi^1_0$ and mean-field $\mu^1_0$ are uniform over actions and states, respectively. The initial estimate of the $Q$-function $Q^1_1$ is zero and transition probability estimate $\hat P^1_1$ is uniform. By choosing learning coefficient $c_\beta = 5$, learning rate $\nu = 0.55$ and episodic length $T=5\times10^4$ we observe in Figure \ref{fig:tran_prob_Q_func_conv} that $Q$-learning and transition probability estimation converge very well inside episode $k = 1$.

\begin{figure}[h!]
    \centering
    \includegraphics[width=0.85\linewidth]{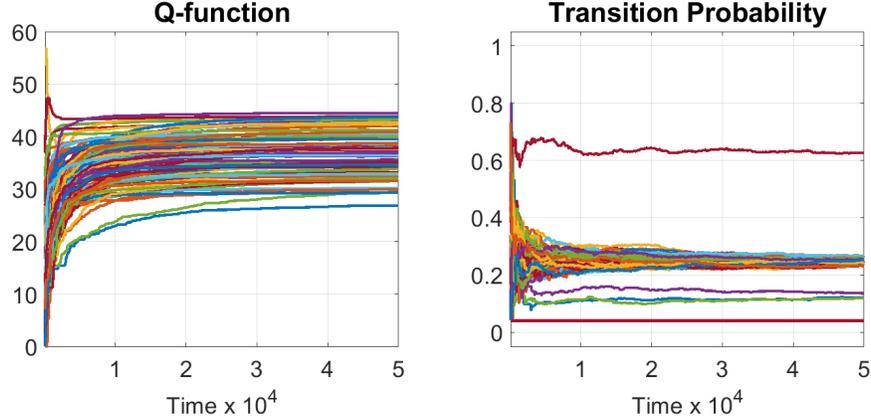}
    \caption{Convergence of transition probability and $Q$-function estimation for episode $k=1$}
    \label{fig:tran_prob_Q_func_conv}
\end{figure}
Furthermore, by choosing $c_\mu = c_\pi = 0.5$ and the two time-scale learning rates as $\theta = 0.55 < \gamma = 0.6$ we see that the control policy and mean-field estimates converge after $K=300$ many episodes in Figure \ref{fig:mf_cp_conv}. For this particular simulation we forego the projection step as it does not have a significant impact on the convergence of the algorithm.

\begin{figure}[h!]
    \centering
    \includegraphics[width=0.85\linewidth]{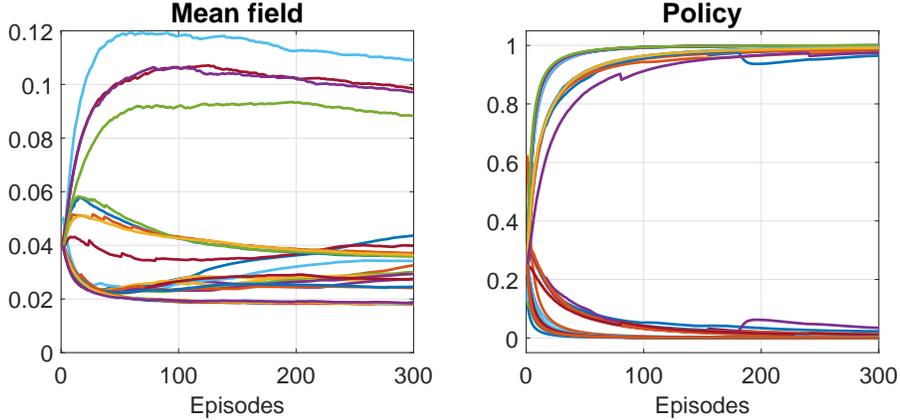}
    \caption{Convergence of mean-field and control policy estimates}
    \label{fig:mf_cp_conv}
\end{figure}
Next, we consider the setting where the state space consists of two communicating classes, thereby nullifying the communicating MDP assumption (Assumption \ref{asm:comm}). In particular, the states $\Cs^1:=\{(4,5),(5,4),(5,5)\}$ form a communicating class which is not closed and the rest of the state space is a closed communicating class $\Cs^0$. The reward function $r(s,\mu)$ is exactly the same as the previous simulation but the transition probabilities are modified to prevent transition $\Cs^0 \rightarrow \Cs^1$, ensuring that $\Cs^0$ is closed and $\Cs^1$ is not closed. The learning rates and other estimates are initialized as before. Figure \ref{fig:mf_cp_conv_no_comm} shows the convergence of mean-field and control policy estimates in the Sandbox learning algorithm. 

\begin{figure}[h!]
    \centering
    \includegraphics[width=0.85\linewidth]{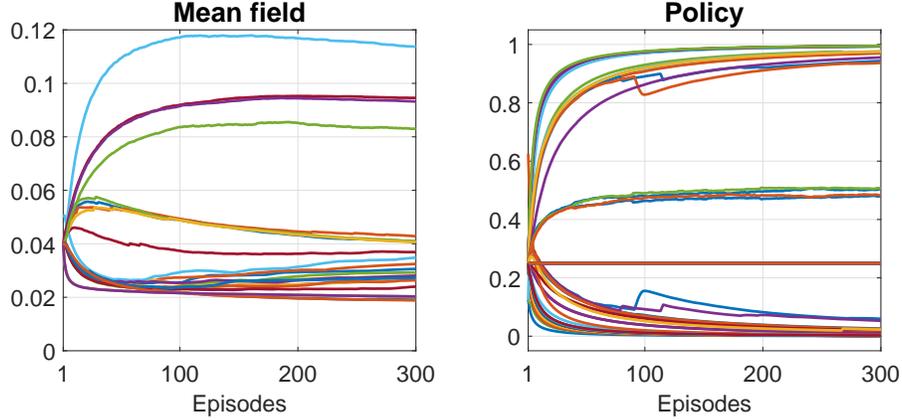}
    \caption{Convergence of mean-field and control policy estimates in the absence of Assumption \ref{asm:comm}}
    \label{fig:mf_cp_conv_no_comm}
\end{figure}
The simulation shows good convergence properties of the algorithm even in absence of the communicating MDP condition. This is due to the fact that the MC transitions (in finite time) from $\Cs^1$ into $\Cs^0$ and stays there, allowing good approximations of $\Gamma_1,\Gamma_2$ operators, resulting in convergence. If there were multiple closed classes with initial distribution spread amongst them, the algorithm's performance would be variable.

%% file: Sec_discussion.tex
In this paper we have developed the Sandbox learning algorithm with finite-time guarantees to approximate the stationary MFE of a 
MFG without access to a mean-field oracle, using the single sample path of the generic agent. The sample complexity of the Sandbox learning algorithm is $\Os(\epsilon^{-4})$ where $\epsilon$ is the MFE approximation error. 
The proof of convergence has relied on goodness of transition probability and $Q$-function estimates (along slowly time-varying MC). The control policy and the mean-field were then updated using two time-scale learning rates and approximate consistency and optimality operators. We also generalize the covering time and ergodicity assumptions in literature by proposing a communicating MDP assumption. This work opens up several interesting research directions. It would be worthwhile to explore finite sample bounds for Sandbox learning under a \emph{weakly} communicating MDP condition. Another important research direction would be to explore  how feature embeddings can improve the scalability beyond the tabular MFG setting. Furthermore, extending oracle-less learning to the benchmark setting of Linear Quadratic MFGs may also help in solving several real-world scenarios such as, consensus and formation flying.

%% file: sec_Appendix.tex
\section*{\centering \LARGE{Appendix}}
\section{EXTENDED LITERATURE}
\label{sec:lit_rev}
\input{Sec_Lit_review}

\section{Proofs of Results in Section~\ref{sec:Sandbox_RL_Analysis}}

Throughout this section the cardinalities of the state and action spaces are denoted by $S = |\Ss|$ and $A = |\As|$, respectively.

\subsection{Proof of Lemma \ref{lem:suff_exp}}
\begin{proof}
Let us consider exploration noise coefficient of the form $\psi^k_t = \psi$ for some $\psi \in (0,1-c_\pi)$, and $c_\pi$ and $\theta$ defined in Section \ref{subsec:sandbox_RL}. 
The proof consists of two parts; in the first part we show that under the policy update \eqref{eq:algorithm_iteration2} using learning coefficient $\psi$ the probability of any action $a \in \As$ for any state $s \in \Ss$, will have a uniform lower bound $\pi^k_t(a \mid s) \geq \underline{\pi} > 0$ for all $k \in [K], 1 < t \leq T$. In the second step, we show that using this uniform lower bound and the communicating MDP assumption (Assumption \ref{asm:comm}), the sufficient exploration condition as shown in the statement of Lemma \ref{lem:suff_exp} holds.

To prove the lower bound $\underline\pi$, we start by proving a lower bound on $\pi^k_t(a \mid s)$ for a given episode $k \in [K]$ and $1 \leq t \leq T$. Recalling the update \eqref{eq:algorithm_iteration2} for $1 < t \leq T$
\begin{align*}
\pi^k_{t} & = (1-c^k_{\pi,t}) \pi^k_{t-1} + c^k_{\pi,t} \big((1-\psi)\softmax{\lambda}{\cdot,Q^k_t} + \psi \mathds{1}_{|\As|}  \big)
\end{align*}
where the last part on the right hand side is the uniform exploration noise $\mathds{1}_{|\As|} = (\frac{1}{|\As|},\ldots,\frac{1}{|\As|})$. This can be written as follows:
\begin{align} \label{eq:cp_update_recur}
\pi^k_{t} & = \tilde{c}^k_{\pi,[1,t]} \pi^k_0 + \sum_{l=1}^t c^k_{\pi,[l,t]} ( 1-\psi) \softmax{\lambda}{\cdot,Q^k_l} + \sum_{l=1}^t c^k_{\pi,[l,t]} \psi  \mathds{1}_{|\As|},
\end{align}
where
\begin{align*}
\tilde{c}^k_{\pi,[l,t]} = \prod_{s=l}^t (1 - c^k_{\pi,s}), \hspace{0.25cm} c^k_{\pi,[l,t]} = c^k_{\pi,l} \prod_{s=l+1}^t (1 - c^k_{\pi,s}).
\end{align*}
From the control policy update \eqref{eq:cp_update_recur} we analyze the last part of the right-hand-side. In particular, we will prove that $\sum_{l=2}^{t+1} c^k_{\pi,[l,t+1]} \geq \sum_{l=2}^{t} c^k_{\pi,[l,t]}$ for $1 \leq t \leq T$. First we see that,
\begin{align} \label{eq:c_k_pi_inter}
    c^k_{\pi,[l,t]} = c^k_{\pi,l} (1 - c^k_{\pi,t}) \prod_{s=l+1}^{t-1} (1 - c^k_{\pi,s}) = (1 - c^k_{\pi,t}) c^k_{\pi,[l,t-1]}
\end{align}
Using this equality, we can show
\begin{align}
    \sum_{l=1}^{t+1} c^k_{\pi,[l,t+1]} & = c^k_{\pi,[t+1,t+1]} + \sum_{l=1}^{t} c^k_{\pi,[l,t+1]}  = c^k_{\pi,t+1} + (1 - c^k_{\pi,t+1}) \sum_{l=1}^{t} c^k_{\pi,[l,t]} \nonumber \\
    & = c^k_{\pi,t+1} \bigg( 1 - \sum_{l=1}^{t} c^k_{\pi,[l,t]} \bigg) + \sum_{l=1}^{t} c^k_{\pi,[l,t]} \geq \sum_{l=1}^{t} c^k_{\pi,[l,t]} \label{eq:upper_bound_noise}
\end{align}
where the second equality is obtained using \eqref{eq:c_k_pi_inter}. In the inequality we have used the fact that $\sum_{l=1}^{t} c^k_{\pi,[l,t]} \leq 1$ for all $1 \leq t \leq T$ which can be shown using inductive reasoning. The base case is true since $c^k_{\pi,[1,1]} = c^k_{\pi,1} = \frac{c_\pi}{k^\theta} < 1$ since $c_\pi \leq 1$ and $\zeta > 1$. Now assume that $\sum_{l=1}^{t} c^k_{\pi,[l,t]} \leq 1$; then,
\begin{align*}
\sum_{l=1}^{t+1} c^k_{\pi,[l,t+1]} & = c^k_{\pi,t+1} + (1 - c^k_{\pi,t+1}) \sum_{l=1}^{t} c^k_{\pi,[l,t]} \leq c^k_{\pi,t+1} + 1 - c^k_{\pi,t+1} = 1
\end{align*}

\vspace{0cm}

Hence we have proved that $\sum_{l=2}^{t} c^k_{\pi,[l,t]} \geq c^k_{\pi,[1,1]} = c^k_{\pi,1}$ for all $1 \leq t \leq T$. Let us define $\underline\pi^k_t := \min_{a \in \As, s \in \Ss} \pi^k_t(a \mid s)$ as the lower limit on exploration noise in policy $\pi^k_t$ for $1 \leq t \leq T$. Using \eqref{eq:cp_update_recur} and the definition of $\underline \pi^k_t$ and we can deduce
\begin{align*}
    \underline\pi^k_t \geq \tilde c^k_{\pi,[1,t]} \underline\pi^k_0 + \sum_{l=1}^t c^k_{\pi,[l,t]} \psi \geq \tilde c^k_{\pi,[1,t]} \underline\pi^k_0 + c^k_{\pi,1} \psi/|\As|
\end{align*}
for $1 \leq t \leq t$ where the second inequality uses $\sum_{l=1}^{t} c^k_{\pi,[l,t]} \geq c^k_{\pi,[1,1]} = c^k_{\pi,1}$ for all $1 \leq t \leq T$.
This gives us a $k$-dependent lower bound on $\underline\pi^k_t$. Next we convert it into a uniform lower bound independent of $k$. Let us define a uniform lower bound dependent on $k$, $\underline\pi^k := \min_{1 < t \leq T} \underline\pi^k_t$. Using \eqref{eq:algorithm_iteration2}, we can write 
\begin{align*}
\pi^k_{1} & = (1-c^k_{\pi,1}) \pi^k_0 + c^k_{\pi,1} \big( ( 1-\psi ) \softmax{\lambda}{\cdot,Q^k_1} + \psi \mathds{1}_{|\As|} \big) \\
& = (1-c^k_{\pi,1}) \pi^{k-1}_T + c^k_{\pi,1} \big( ( 1-\psi ) \softmax{\lambda}{\cdot,Q^k_1} + \psi \mathds{1}_{|\As|} \big)
\end{align*}

Using this relation and the definition of $\underline\pi^k$, we deduce that
\begin{align*}
\underline\pi^k \geq (1-c^k_{\pi,1}) \underline\pi^{k-1} + c^k_{\pi,1} \psi/|\As|  
\end{align*}
Using recursion we can write this as
\begin{align} \label{eq:under_pi_k}
\underline\pi^k \geq \tilde\vartheta_{1,k} \underline\pi^1 + \sum_{l=1}^k \vartheta_{l,k} \psi/|\As|, \hspace{0.1cm} \text{where} \hspace{0.1cm} \tilde\vartheta_{l,k} = \prod_{s=l}^k (1-c^s_{\pi,1}) \hspace{0.1cm} \text{and} \hspace{0.1cm} \vartheta_{l,k} = c^l_{\pi,1} \tilde\vartheta_{l+1,k}
\end{align}
Next we will show that $\sum_{l=1}^k \vartheta_{l,k} \psi/|\As| \geq \vartheta_{1,1} \psi/|\As| = c_{\pi} \psi/|\As|$ for all $k \in [K]$ which coupled with \eqref{eq:under_pi_k} naturally leads to the conclusion that $\underline\pi^k \geq c_{\pi} \psi/|\As|$. By definition, we have
\begin{align*}
    \sum_{l=1}^k \vartheta_{l,k} \psi/|\As| & = \vartheta_{k,k} \psi/|\As| + \sum_{l=1}^{k-1} \vartheta_{l,k} \psi/|\As|  = c^k_{\pi,1} \psi/|\As| + (1-c^k_{\pi,1}) \sum_{l=1}^{k-1} \vartheta_{l,k-1} \psi/|\As| \\
    & = c^k_{\pi,1} \bigg( \psi -  \sum_{l=1}^{k-1} \vartheta_{l,k-1} \psi\bigg)/|\As| + \sum_{l=1}^{k-1} \vartheta_{l,k-1} \psi/|\As| \geq \sum_{l=1}^{k-1} \vartheta_{l,k-1} \psi/|\As|
\end{align*}
where the last inequality follows from the fact that $\sum_{l=1}^{k} \vartheta_{l,k} \leq 1$ for all $k \in [K]$ which can be shown using inductive reasoning. The base case is true since $\vartheta_{1,1} = c_\pi < 1$. Now assume that $\sum_{l=1}^{k} \vartheta_{l,k} \leq 1$; then
\begin{align*}
\sum_{l=1}^{k+1} \vartheta_{l,k+1} & = \vartheta_{k+1,k+1} + (1 - c^{k+1}_{\pi,1}) \sum_{l=1}^{k} \vartheta_{l,k} \leq c^{k+1}_{\pi,1} + 1 - c^{k+1}_{\pi,1} = 1
\end{align*}
We have proved that  $\sum_{l=1}^k \vartheta_{l,k} \psi \geq \vartheta_{1,1} \psi = c_{\pi} \psi$ for all $k \in [K]$, and hence using \eqref{eq:under_pi_k} we can deduce that $\underline\pi^k \geq c_{\pi} \psi/|\As|$ which implies $\pi^k_t(a|s) \geq \underline\pi := c_{\pi} \psi/|\As| > 0$ for any $s \in \Ss$ and $a \in \As$.

Next we show how using this uniform lower bound $\underline\pi$ along with the communicating MDP assumption (Assumption \ref{asm:comm}) sufficient exploration can be proved. From Assumption \ref{asm:comm} for any $i,j \in \Ss$ and $\tilde\mu \in S(\epsilon^\net)$ we know that there exists a finite horizon $H(\tilde\mu)$ and a set of actions $\tilde a_1, \ldots, \tilde a_{H(\tilde\mu)}$ such that
\begin{align} \label{eq:P_mu_tilde}
P^{\tilde\mu}_{ij} := P\big(s_{H(\tilde\mu)} = j \mid a_1=\tilde{a}_1,\ldots,a_{H(\tilde\mu)}=\tilde{a}_{H(\tilde\mu)}, s_1 = i, \mu = \tilde\mu\big) > 0.
\end{align}
Let us define $\underline{P} := \min_{i,j \in \Ss, \tilde{\mu} \in S(\epsilon^\net)} P^{\tilde{\mu}}_{ij}$ and due to the finiteness of $\Ss$ and $S(\epsilon^\net)$ we can guarantee $\underline{P} > 0$. Due to the Lipschitzness of stochastic kernel $P$ with respect to $\mu$ for any $s \in \Ss, a \in \As$, and $\mu,\mu' \in \Delta(\Ss)$,
\begin{align*}
    \lVert P(\cdot \mid s,a,\mu) - P(\cdot \mid s,a,\mu') \rVert_1 \leq L_\mu \lVert \mu - \mu'\rVert_1
\end{align*}
for Lipschitz constant $L_\mu > 0$. This allows us to lower bound $P(s' | s,a,\mu^k_t)$ where $\mu^k_t$ evolves according to the update rule \eqref{eq:algorithm_iteration1}. Hence for any $s,s' \in \Ss, a \in \As$ and $\mu^k_t \in \Delta(\Ss)$,
\begin{align}
P(s' | s,a,\mu^k_t) & \geq P(s' | s,a,\mu^k_1) - |P(s' | s,a,\mu^k_t) - P(s' | s,a,\mu^k_1)| \nonumber \\
& \geq P(s' | s,a,\mu^k_1) - L_\mu \sum_{k=2}^\infty c^k_{\mu,k} \nonumber \\
& \geq P(s' | s,a,\mu^k_1) - \frac{L_\mu c_\mu}{2^\zeta} \label{eq:stoch_ker_dev}
\end{align}
and we know that $\mu^k_1 \in S(\epsilon^\net)$ due to the projection step in \eqref{eq:algorithm_iteration1}. As a result, the change in stochastic kernel $P(\cdot \mid s, a, \mu^k_t)$ can be bounded by controlling the inter-episodic learning coefficient $\zeta$. Let us define the probability of reaching state $j \in \As$ from state $i \in \Ss$ in time $t \geq H:= \max_{\tilde\mu \in S(\epsilon^\net)} H(\tilde \mu)$ under the stochastic kernels induced by mean-field $(\mu^k_1, \ldots, \mu^k_t)$ as $P(s_t = j | s_1 = i, (\mu^k_1, \ldots, \mu^k_t))$. From Assumption \ref{asm:comm} we know that there exists a set of actions $(\tilde a_1, \ldots, \tilde a_t)$ such that $P(s_t = j | s_1 = i, (a_1 = \tilde a_1, \ldots, a_t = \tilde a_t),(\mu^k_1, \ldots, \mu^k_1)) \geq \underline P$. Using these facts we can deduce,
\begin{align*}
    P(s_t = j | s_1 & = i, (\mu^k_1, \ldots, \mu^k_t)) \geq P(s_t = j | s_1 = i, (a_1 = \tilde a_1, \ldots, a_t = \tilde a_t),(\mu^k_1, \ldots, \mu^k_t)) \\
    & = \prod_{l = 1}^{t-1} \sum_{s_1 = i, s_l \in \Ss, s_k = j} P(s_{l+1} = \tilde s_{l+1} | s_{l} = \tilde s_{l}, a_l = \tilde{a}_l, \mu^k_l) \pi^k_l(a_l = \tilde a_l | s_l = \tilde s_l) \\
    & \geq \prod_{l = 1}^{t-1} \sum_{s_1 = i, s_l \in \Ss, s_k = j} P(s_{l+1} = \tilde s_{l+1} | s_{l} = \tilde s_{l}, a_l = \tilde{a}_l, \mu^k_l) \underline\pi \\
    & \geq \prod_{l = 1}^{t-1} \Big(\sum_{s_1 = i, s_l \in \Ss, s_k = j} P(s_{l+1} = \tilde s_{l+1} | s_{l} = \tilde s_{l}, a_l = \tilde{a}_l, \mu^k_1) - |\Ss|\frac{L_\mu c_\mu}{2^\zeta} \Big)\underline\pi \\
    & \geq \prod_{l = 1}^{t-1} \sum_{s_1 = i, s_l \in \Ss, s_k = j} P(s_{l+1} = \tilde s_{l+1} | s_{l} = \tilde s_{l}, a_l = \tilde{a}_l, \mu^k_1) - C(t)|\Ss|\frac{L_\mu c_\mu}{2^\zeta} \underline\pi \\
    & \geq \underline P - C(2H)|\Ss|\frac{L_\mu c_\mu}{2^\zeta} \underline\pi^{2H} \geq  \underline P/2
\end{align*}
The first inequality is obtained using the fact that under exploration noise $\psi^k_t = \psi \in (0,1-c_\pi)$, $\pi^k_t(a | s)$ is lower bounded by $\underline\pi$. The second inequality follows from Lipschitzness of stochastic kernel $P$ and inter-episodic learning coefficient $\zeta > 1$ as shown in \eqref{eq:stoch_ker_dev}. The third inequality is obtained by using the fact that probability $P(s' | s,a,\mu) \leq 1$ and for $\zeta$ high enough $|\Ss| \frac{L_\mu c_\mu}{2^\zeta} \leq 1$. The fourth inequality uses the lower bound on $P^{\tilde \mu}_{ij}$ for $i,j \in \Ss$ and $\tilde \mu \in S(\epsilon^\net)$ as shown in \eqref{eq:P_mu_tilde}. The final bound uses the fact that for $\zeta$ large enough $C(2H)|\Ss|\frac{L_\mu c_\mu}{2^\zeta} \underline\pi^{2H} \leq \underline P/2$.
\end{proof}

\subsection{Proof of Lemma \ref{lem:tran_prob_conv}}
\begin{proof}
In this proof we provide finite sample convergence bounds for the transition probability estimation \eqref{eq:consistency_estimator} under the slowly time-varying MC setting. The proof of Lemma \ref{lem:tran_prob_conv} relies on Freedman's inequality \cite{freedman1975tail} similar to the analysis of Theorem 4 in \cite{hsu2019mixing}. Our lemma generalizes transition probability estimation for the slowly time-varying MC setting. The proof relies on introducing a stochastic process $Y_t$ (dependent on visitation of a fixed pair of states $i,j$) which is shown to be a Martingale difference sequence. The transition probability estimation error is shown to be a function of the sum of $Y_t$ and a drift term due to the slowly time-varying MC setting. The drift term is shown to be small due to the Lipschitz property of transition dynamics and the slowly time-varying MC. Then, using Freedman's inequality, we show that the estimation error is monotonically decreasing with the visitation number of the pair of states $i,j$. Finally, we prove a high confidence lower bound on the visitation number of any pair of states $i,j$ under the sufficient exploration condition (Lemma \ref{lem:suff_exp}), yielding our convergence result.


	We recall the definition of $\epsilon^k_P := \lVert \hat{P}^k_{T} - P^k_1 \rVert_F$ where we use $P^k_t$ as a shorthand for $P_{\pi^k_t,\mu^k_t}$. We use Freedman's inequality to obtain the estimation error for estimator $\hat{P}^k_{T}$ as in \cite{hsu2019mixing}. Furthermore, since we are dealing with a single episode $k$ we suppress the use of episode $k$ for clarity. Let $\Fs_t$ be the $\sigma$-field generated by $\{s_1,\mu_1,\pi_1, \ldots, s_t,\mu_t,\pi_t\}$. Let us start by fixing a pair of states $(i,j)$ for any $i,j \in \Ss$. Next let us define a stochastic process $Y_t$ such that $Y_1:=0$ and for $t \geq 2$
	\begin{align*}
		Y_t := \bbone\{s_{t-1} = i\} \big(\bbone\{s_t = j\} - P_{t-1}(i,j) \big)
	\end{align*}
	where $P_{t-1}(i,j)$ is the transition probability of going from state $i$ to $j$ from time $t-1$ to $t$. 
	The stochastic process $(Y_t)_{t \in [T]}$ is a Martingale Difference Sequence since $Y_t$ is $\Fs_t$-measurable, and for $t \geq 2$
	\begin{align*}
		\EE[Y_t \mid \Fs_{t-1}] & = \EE \big[ \bbone\{s_{t-1} = i\} \big(\bbone\{s_t = j\} - P_{t-1}(i,j) \big) \mid \Fs_{t-1} \big], \\
		& = \bbone\{s_{t-1} = i\} \Big( P_{t-1}(i,j) -  P_{t-1}(i,j)\Big) = 0.
	\end{align*}
	Furthermore, $\forall t \in [T]$, $Y_t \in [-P_{t-1}(i,j),1-P_{t-1}(i,j)] \subset [-1,1]$.  Summing up $Y_t$ for $t \in [T]$,
	\begin{align}
		S_T := \sum_{t=1}^T Y_t & = \sum_{t=2}^T \bbone\{s_{t-1} = i\} \big(\bbone\{s_t = j\} - P_{t-1}(i,j) \big), \nonumber \\
		& = \sum_{t=2}^T \bbone\{s_{t-1} = i\} \big(\bbone\{s_t = j\} - P_{1}(i,j) + P_{1}(i,j) - P_{t-1}(i,j) \big), \nonumber \\
		& = N_{i,j} - N_i P_{1}(i,j) + \sum_{t=2}^T \bbone\{s_{t-1} = i\} \tilde{P}_{t-1}(i,j), \label{eq:S_T}
	\end{align}
	where $\tilde{P}_t := P_1 - P_t$ is the drift in the true transition probability. For use in Freedman's inequality, consider the process
	\begin{align*}
		&\EE[Y_t^2 \mid \Fs_{t-1}] = \EE \big[ \bbone\{s_{t-1} = i\} \big( \bbone\{ s_t = j \} P_{t-1}(i,j) + P_{t-1}^2(i,j) \big) \mid \Fs_{t-1} \big], \\
		& = \bbone\{s_{t-1} = i\} P_{t-1}(i,j) \big(1 -  P_{t-1}(i,j)  \big), \\
		& = \bbone\{s_{t-1} = i\} \big( P_{1}(i,j) - \tilde P_{t-1}(i,j) -  P^2_{1}(i,j) + 2 P_{1}(i,j) \tilde P_{t-1}(i,j) - \tilde P^2_{t-1}(i,j)  \big), \\
		& = \bbone\{s_{t-1} = i\} P_{1}(i,j) \big( 1 -  P_{1}(i,j) \big) + \bbone\{s_{t-1} = i\} \tilde P_{t-1}(i,j) \big(2 P_{1}(i,j)  -1  - \tilde P_{t-1}(i,j)  \big).
	\end{align*}
	Since $\EE[Y_t^2 \mid \Fs_{t-1}] \geq 0$, both terms in the above expression are positive. Hence its summation $V_T$ will be
	\begin{align}
		V_T & := \sum_{t=2}^T \EE[Y_t^2 \mid \Fs_{t-1}] =  N_i P_{1}(i,j) \big( 1 -  P_{1}(i,j) \big) + \sum_{t=2}^T \bbone\{s_{t-1} = i\} \tilde P_{t-1}(i,j) \big(2 P_{1}(i,j)  -1  - \tilde P_{t-1}(i,j)  \big). \label{eq:V_T}
	\end{align}
	Again both parts of the above expression are positive. Recalling \eqref{eq:P_est} we write the estimation error as
	\begin{align*}
		\hat{P}_{T}(i,j) - P_1(i,j) & = \frac{N_{i,j} - N_i P_1(i,j) + 1/S - P_1(i,j)}{N_i+1}, \\
		& = \frac{N_{i,j} - N_i P_1(i,j)}{N_i + 1} + \frac{1/S - P_1(i,j)}{N_i + 1}, \\
		& = \frac{S_T - \sum_{t=2}^T \bbone\{s_{t-1} = i\} \tilde{P}_{t-1}(i,j)}{N_i + 1} + \frac{1/S - P_1(i,j)}{N_i + 1},
	\end{align*}
	where the last equality is due to \eqref{eq:S_T}. Applying Corollary 1 from \cite{hsu2019mixing}, which is based on Freedman's inequality, we get
	\begin{align}
		& \lvert \hat{P}_{T}(i,j) - P_1(i,j) \rvert \nonumber \\
		& \leq \sqrt{\frac{2c V_T \tau_{T,\delta_P}}{(N_i + 1)^2}} + \frac{4 \tau_{T,\delta_P} + \sum_{t=2}^T \bbone\{s_{t-1} = i\} \lvert \tilde{P}_{t-1}(i,j) \rvert +  \lvert 1/S - P_1(i,j) \rvert}{N_i + 1}, \nonumber \\ 
		& = \bigg( \frac{2c N_i P_{1}(i,j) \big( 1 -  P_{1}(i,j) \big)  \tau_{T,\delta_P} }{(N_i + 1)^2} \nonumber \\
		& \hspace{2.7cm}+ \frac{2c \sum_{t=2}^T \bbone\{s_{t-1} = i\} \tilde P_{t-1}(i,j) \big(2 P_{1}(i,j)  -1  - \tilde P_{t-1}(i,j)  \big) \tau_{T,\delta_P}}{(N_i + 1)^2} \bigg)^{\frac{1}{2}}\nonumber \\
		& \hspace{4cm}+ \frac{4 \tau_{T,\delta_P} + \sum_{t=2}^T \bbone\{s_{t-1} = i\} \lvert \tilde{P}_{t-1}(i,j) \rvert +  \lvert 1/S - P_1(i,j) \rvert}{N_i + 1}, \nonumber \\
		& \leq \sqrt{\frac{2c N_i P_{1}(i,j) \big( 1 -  P_{1}(i,j) \big) \tau_{T,\delta_P}}{(N_i + 1)^2}} + \frac{\sqrt{2c \sum_{t=2}^T \lvert \tilde P_{t-1}(i,j) \rvert \tau_{T,\delta_P}}}{N_i + 1} \nonumber \\
		& \hspace{6cm}+ \frac{4 \tau_{T,\delta_P} + \sum_{t=2}^T  \lvert \tilde{P}_{t-1}(i,j) \rvert +  \lvert 1/S - P_1(i,j) \rvert}{N_i + 1}, \nonumber \\
		& \leq \sqrt{\frac{2c \tau_{T,\delta_P}}{N_i + 1}} + \frac{\sqrt{2c \sum_{t=2}^T \lvert \tilde P_{t-1}(i,j) \rvert \tau_{T,\delta_P}}}{N_i + 1}+ \frac{4 \tau_{T,\delta_P} + \sum_{t=2}^T \lvert \tilde{P}_{t-1}(i,j) \rvert +  \lvert 1/S - P_1(i,j) \rvert}{N_i + 1}, \label{eq:P_hat_P_1}
	\end{align}
	with probability at least $1-\delta_P/(2S^2)$, where $\tau_{T,\delta_P} = \Os(\log(\frac{2 S^3 \log(T)}{\delta_P}))$. We used equation \eqref{eq:V_T} to obtain the second inequality. Analyzing $\lvert \tilde P_t(i,j) \rvert$,
	\begin{align}
		\lvert \tilde P_t(i,j) \rvert & \leq \lVert P_1 - P_t \rVert_F = \lVert P_{\pi_1,\mu_1} - P_{\pi_t,\mu_t} \rVert_F, \nonumber \\
		& \leq \sum_{l=1}^{t-1} \big( \lVert P_{\pi_l,\mu_l} - P_{\pi_l,\mu_{l+1}} \rVert_F + \lVert  P_{\pi_l,\mu_{l+1}} -  P_{\pi_{l+1},\mu_{l+1}}\rVert_F \big), \nonumber \\
		&  \leq \sum_{l=1}^{t-1} \big( L^\mu_P \lVert \mu_{l+1} - \mu_l \rVert_1 + L^\pi_P \lVert \pi_{l+1} - \pi_l \rVert_{TV} \big) \nonumber \\
		& \leq \sum_{l=2}^{t} \big( L^\mu_P c_{\mu,l} + L^\pi_P c_{\pi,l} \big), \nonumber \\
		& \leq ( L^\mu_P c_\mu + L^\pi_P c_\pi )\sum_{l=2}^{t} l^{-\zeta}, \nonumber \\
		& \leq \frac{L^\mu_P  c_\mu + L^\pi_P  c_\pi}{\zeta - 1} 2^{1 - \zeta} = \tilde{L}_P 2^{1-\zeta}. \label{eq:P_drift}
	\end{align}
	where $c_{\mu,t}:=c_\mu t^{-\zeta}$, $c_{\pi,t}:=c_\pi t^{-\zeta}$ and $\tilde{L}_P:=10 (L^\mu_P c_\mu + L^\pi_P c_\pi)$ for $\zeta \geq 1.1$. The second inequality above is due to the Lipschitz conditions on $P$ in Lemma \ref{lem:tran_prob_conv} and the third inequality is due to the fact that $\lVert \mu \rVert_1, \lVert \pi \rVert_{TV} \leq 1$ for any $\mu \in \Ps(\Ss)$ and $\pi \in \Ps$. Now the estimation error can be bounded using \eqref{eq:P_hat_P_1} and \eqref{eq:P_drift}:
	\begin{align}
		& \lvert \hat{P}_{T}(i,j) - P_1(i,j) \rvert \nonumber \\
		& \leq \sqrt{\frac{2c \tau_{T,\delta_P}}{N_i + 1}} + \frac{\sqrt{2c \tau_{T,\delta_P}\tilde{L}_P T}}{N_i + 1} 2^{\frac{1-\zeta}{2}}+ \frac{4 \tau_{T,\delta_P} + \tilde{L}_P T 2^{1-\zeta} +  \lvert 1/S - P_1(i,j) \rvert}{N_i + 1}, \label{eq:P_inter_bound}
	\end{align}
	with probability at least $1-\delta_P/(2S^2)$. Next we need to lower bound $N_i$. In the following lemma  we show that due to the sufficient exploration condition, $N_i$ grows at least linearly with $T$ with high probability.
	\begin{lemma} \label{lem:bound_N_i}
		Using Lemma \ref{lem:suff_exp}, $N_i \geq T/T_e$ with probability at least $1 - \delta_P/(2S^2)$, where 
		\begin{align} \label{eq:T_e}
			T_e := \Os \Big( \frac{1}{\sigma} \log \Big( \frac{2S^3}{\delta_P} \Big) \Big).
		\end{align}
	\end{lemma}
	\begin{proof}
		For a fixed state $i \in \Ss$, define event $\Es^k$ such that $\sum_{t = 1}^{k T_e} \bbone \{i_t = i\} \geq k$, for a given integer $k$ such that $1 \leq k \leq K_e := \lceil T/T_e \rceil $. We show that $\Es^K$ is a high probability event given the sufficient exploration condition (Lemma \ref{lem:suff_exp}). 
		For a given $i \in \Ss$
		, define a random variable,
		\begin{align*}
			X^i_t = \II\{ i_t = i \} - \EE[ \II \{i_t = i\} \mid \Fs_{t-\tau}]
		\end{align*}
		This random variable is an $\Fs_t$ adapted process with $\EE[X^i_t \mid \Fs_{t-\tau}] = 0$ and $\lvert X^i_t \rvert \leq 1$. Let $l$ be an integer $0 \leq l \leq \tau$. For a fixed $l$, define the process $Y^i_{l,k} = X^i_{k\tau + l}$ and define filtration $\tilde \Fs_{l,k} := \Fs_{k \tau+l}$. We can deduce that
		\begin{align*}
			\EE[Y^i_{l,k} \mid \tilde \Fs_{l,k-1}] = \EE [X^i_{k \tau+l} \mid \Fs_{k\tau+l - \tau}] = 0, \lvert Y^i_{l,k} \rvert \leq 1,
		\end{align*} 
		and $Y^i_{l,k}$ is $\tilde \Fs_{l,k}$ measurable. 
		Combining these facts, $Y^i_{l,k}$ is a Martingale Difference Sequence. Using Azuma-Hoeffding inequality and Lemma \ref{lem:suff_exp} we can deduce that for a given $i \in \Ss$ and $k = K_e$ where $T_e:=\Os( \ln(2S^3/\delta_P)/\sigma)$, $\sum_{t=1}^{T} \II\{i_t = i\} \geq K_e$ with probability at least $1-\delta_P/(2S^3)$. Taking a union bound over all $i \in \Ss$, we get $\sum_{t=0}^{T} \II\{i_t = i\} \geq K_e, \forall i \in \Ss$ with probability at least $1-\delta_P/(2 S^2)$. 
	\end{proof}
	\noindent Using \eqref{eq:P_inter_bound} and Lemma \ref{lem:bound_N_i} the estimation error can be written as
	\begin{align*}
		& \lvert \hat{P}_{T}(i,j) - P_1(i,j) \rvert \\
		& \leq \sqrt{\frac{2c \tau_{T,\delta_P} T_e}{T}} + \sqrt{\frac{2c \tau_{T,\delta_P}\tilde{L}_P}{T}} T_e 2^{\frac{1-\zeta}{2}}+ \frac{(4 \tau_{T,\delta_P} +  \lvert 1/S - P_1(i,j) \rvert) T_e}{T} + \tilde{L}_P T_e 2^{1-\zeta},
	\end{align*}
	with probability at least $1-\delta_P/S^2$. Using a union bound over all pairs $(i,j) \in \Ss \times \Ss$, the definition of Frobenius norm and the equivalence between $1$ and $2$ vector norms,
	\begin{align*}
		\lVert \hat{P} - P_1 \rVert_F  = \tilde\Os(T^{-1/2}) + \tilde\Os(T^{-1}) + \Os(2^{1-\zeta}).
	\end{align*}
	with probability at least $1-\delta_P$.	Hence we have completed the proof.
\end{proof}

\subsection{Proof of Lemma \ref{lem:Q_learn_conv}}
\begin{proof}
In this proof we provide finite sample convergence bounds for the $Q$-learning update \eqref{eq:Q_learn} within the slowly time-varying MC setting. The proof of Lemma \ref{lem:Q_learn_conv} follows an approach similar to proof of Theorem 4 in \cite{qu2020finite} and extends the results to a slowly time-varying MC and learning exponent $0.5 < \nu \leq 1$, which is empirically observed to have better convergence properties. We also find that convergence cannot be guaranteed for $0 < \nu \leq 0.5$. The proof starts with breaking down the error $\epsilon^k_Q$ into several components. Then we obtain bounds on those components by proving certain properties like boundedness and the Martingale Difference Sequence property. Following that we prove that the error accumulated due to the slowly time-varying MC setting is small due to the Lipschitz properties of the transition probability and reward function. Finally combining all these results, the total error itself is shown to be converging using the contraction mapping property of the discounted Bellman update. 

This proof uses the fact that the coefficient of the learning rate $c_\beta$ in the $Q$-learning update \eqref{eq:Q_learn} is lower bounded by $\frac{1}{\sigma}\max \Big\{{\nu+\zeta-1,\frac{1}{(1-\sqrt{\rho})}} \Big\}$. In this proof we suppress the use of superscript $k$ since we are dealing with a single episode. Recall the definition of $\epsilon_Q := \lVert Q_{T} - Q^*_1 \rVert_\infty$ where $Q^*_1 := Q^*_{\mu_1}$ the optimal $Q$-function for the MDP induced by mean-field $\mu_1$. The $Q$-learning update can be written down as:
	\begin{align}
		Q_{t+1} = Q_t + \beta_t [e_{i_t}^T[F(\mu_t,Q_t) - Q_t] + w(t,\mu_t)]e_{i_t} \label{eq:Q_update}
	\end{align}
	where $\beta_t = \frac{c_\beta}{(t+1)^\nu}$ and
	\begin{align}
		w(t,\mu_t) & = \rho [\max_{a \in \As} Q_t(s_{t+1},a) - \EE_{s' \sim P(\cdot \mid s_t,a_t,\mu_t)}[\max_{a' \in \As}Q_t(s',a')], \label{eq:noise_Q_def}\\
		F(\mu_t,Q_t)(s,a) & = r(s,a,\mu_t) + \rho \EE_{s' \sim P(\cdot \mid s_t,a_t,\mu_t)}[\max_{a' \in \As}Q_t(s',a')] \nonumber
	\end{align}
	
	The noise $w(\cdot,\cdot)$ is bounded by $\bar{w}$, is measurable with respect to $\Fs_{t+1}$ and $\EE[w(t,\mu_t) \mid \Fs_t] = 0$. We further decompose the update rule using $D_t := \EE[e_{i_t}^T e_{i_t} \mid \Fs_{t-\tau}]$. The matrix $D_t$ is a diagonal matrix with elements $(d_{t,i})_{i \in \Ss \times \As}$, where $d_{t,i} = \PP(i_t = i \mid \Fs_{t-\tau})$, and from the sufficient exploration condition (Lemma \ref{lem:suff_exp}) we know that $d_{t,i} \geq \sigma > 0$.
	\begin{align*}
		Q_{t+1} & = Q_t + \beta_t D_t (F(\mu_t,Q_t) - Q_t) + \beta_t (e^T_{i_t} e_{i_t} - D_t)(F(\mu_t,Q_t) - Q_t) + \beta_t w(t,\mu_t) e_{i_t}, \\
		& = Q_t + \beta_t D_t (F(\mu_t,Q_t) - Q_t)  +\beta_t (e^T_{i_t} e_{i_t} - D_t)(F(\mu_{t-\tau},Q_{t-\tau}) - Q_{t-\tau}) + \beta_t w(t,\mu_t) e_{i_t} \\
		& \hspace{4cm} + \beta_t (e^T_{i_t} e_{i_t} - D_t)(F(\mu_t,Q_t) - F(\mu_{t-\tau},Q_{t-\tau}) - Q_t + Q_{t-\tau})
	\end{align*}
	Let us define
	\begin{align*}
		\epsilon_t & := (e^T_{i_t} e_{i_t} - D_t)(F(\mu_{t-\tau},Q_{t-\tau}) - Q_{t-\tau}) + \beta_t w(t,\mu_t) e_{i_t}, \\
		\phi_t & := (e^T_{i_t} e_{i_t} - D_t)(F(\mu_t,Q_t) - F(\mu_{t-\tau},Q_{t-\tau}) - Q_t + Q_{t-\tau})
	\end{align*}
	The process $\epsilon_t$ is $\Fs_{t+1}$ measurable and
	\begin{align*}
		& \EE[\epsilon_t \mid \Fs_{t-\tau}] \\
		& = \EE[e_{i_t}^T e_{i_t} - D_t \mid \Fs_{t-\tau}](F(\mu_{t-\tau},Q_{t-\tau}) - Q_{t-\tau}) + \EE[\EE[w(t,\mu_t) \mid \Fs_t] e_{i_t} \mid \Fs_{t-\tau} ] = 0.
	\end{align*}
	Hence, $\epsilon_t$ is a shifted Martingale Difference Sequence. Writing down the $Q$-function as a sum from $\tau$ (Lemma \ref{lem:suff_exp}) to $t$, we get
	\begin{align} \label{eq:Q_update_1}
		Q_{t+1} = \tilde B_{\tau-1,t} Q_{\tau} + \sum_{k=\tau}^{t} B_{k,t} F(\mu_k,Q_k) + \sum_{k=\tau}^{t} \beta_t \tilde B_{k,t}(\epsilon_k + \phi_k),
	\end{align}
	where $B_{k,t} = \beta_k D_k \prod_{l=k+1}^t (I - \beta_l D_l), \tilde B_{k,t} = \prod_{l=k+1}^t (I - \beta_l D_l)$, and $B_{k,t}$ and $\tilde B_{k,t}$ are diagonal matrices composed of elements $b_{k,t,i}$ and $\tilde b_{k,t,i}$ respectively. We also define $\beta_{k,t}$ and $\tilde \beta_{k,t}$ such that
	\begin{align*}
		\beta_{k,t} := \beta_k \prod_{l=k+1}^t (1 - \beta_l \sigma) \geq b_{k,t,i}, \hspace{1cm} \tilde \beta_{k,t} := \prod_{l=k+1}^t (1 - \beta_l \sigma) \geq \tilde b_{k,t,i}
	\end{align*}
	Next we compute the optimality gap $e^Q_t = \lVert Q_t - Q^*_{t} \rVert_\infty$, where $Q^*_{t}$ is the fixed point of the operator $F(\mu_t,\cdot)$.
	\begin{lemma} \label{lem:error_decomp}
		\begin{align} 
			e^Q_{t+1} & \leq \tilde B_{\tau-1,t} e^Q_{\tau} + \rho \max_{i} \sum_{k = \tau}^t b_{k,t,i} e^Q_k + \Big\lVert \sum_{k=\tau}^t \beta_k \tilde B_{k,t} (\epsilon_k + \phi_k) \Big\rVert_\infty \nonumber \\
			& \hspace{5cm} + L^\mu_Q \Big[ \tilde \beta_{\tau-1,t} \sum_{l=\tau}^t c_{\mu,l} + \sum_{k=\tau}^t \beta_{k,t} \sum_{l=k}^t c_{\mu,l} \Big] \label{eq:err_decomp}
		\end{align}
	\end{lemma}
	\begin{proof}
		Using \eqref{eq:Q_update_1} and subtracting $Q^*_{t+1}$ from both sides,
		\begin{align*}
			Q_{t+1} - Q^*_{t+1} & = \tilde B_{\tau-1,t} (Q_{\tau} - Q^*_{t+1}) + \sum_{k=\tau}^{t} B_{k,t} (F(\mu_k,Q_k) - Q^*_{t+1}) + \sum_{k=\tau}^{t} \beta_t \tilde B_{k,t}(\epsilon_k + \phi_k)
		\end{align*}
		Hence we get,
		\begin{align*}
			e^Q_{t+1} & = \tilde B_{\tau-1,t} e^Q_{\tau} + \rho \sup_i \sum_{k=\tau}^t b_{k,t,i} e^Q_k+ \Big\lVert \sum_{k=\tau}^t \beta_k \tilde B_{k,t}(\epsilon_k + \phi_k) \Big\rVert_\infty \\
			& \hspace{5cm} + \Big\lVert \tilde B_{\tau-1,t} (Q^*_{\tau} - Q^*_{t+1}) + \sum_{k=\tau}^t B_{k,t}(Q^*_{k} - Q^*_{t+1}) \Big\rVert_\infty
		\end{align*}
		We can use the Simulation lemma and Lipschitzness of transition probability $P_{\pi,\mu}$ and reward function $R_\mu$ with respect to the mean-field $\mu$ (with corresponding constants $L^\mu_P$ and $L^\mu_R$ respectively), to prove Lipschitzness of $Q^*$ with $\mu$. Due to Lipschizness, we know that for $\mu,\mu' \in \Ps(\Ss)$
		\begin{align*}
			\lVert P_{\pi,\mu} - P_{\pi,\mu'} \rVert_F \leq L^\mu_P \lVert \mu - \mu'\rVert_1, \hspace{0.5cm} \lVert R_\mu - R_{\mu'} \rVert_\infty \leq L^\mu_R \lVert \mu - \mu'\rVert_1
		\end{align*}
		and using the Simulation Lemma \cite{lewis2012optimal} we know that
		\begin{align*}
			\lVert V^*_\mu - V^*_{\mu'} \rVert_\infty \leq \Big( L^\mu_R + \frac{L^\mu_P}{2(1-\rho)} \Big) \lVert \mu - \mu'\rVert_1
		\end{align*}
		where $V^*_\mu$ is the value function of the MDP induced by mean-field $\mu$ and $(1-\rho)^{-1}$ is an upper bound on the value functions due to bounded rewards. Hence
		\begin{align*}
			\lVert Q^*_\mu (s,a) - Q^*_{\mu'}(s,a) \rVert_\infty & = \rho \big( \langle P(\cdot \mid s,a,\mu) , V^*_\mu(\cdot) \rangle - \langle P(\cdot \mid s,a,\mu') , V^*_{\mu'}(\cdot) \rangle  \big), \\
			&  = \rho \big( \langle P(\cdot \mid s,a,\mu) , V^*_\mu(\cdot) \rangle - \langle P(\cdot \mid s,a,\mu) , V^*_{\mu'}(\cdot) \rangle \\
			& \hspace{3cm}+ \langle P(\cdot \mid s,a,\mu) , V^*_{\mu'}(\cdot) \rangle - \langle P(\cdot \mid s,a,\mu') , V^*_{\mu'}(\cdot) \rangle  \big), \\
			& \leq \rho\Big( L^\mu_R + \frac{L^\mu_P}{2(1-\rho)} \Big) \lVert \mu - \mu'\rVert_1 + \rho \frac{L^\mu_P}{2(1-\rho)} \lVert \mu - \mu'\rVert_1 = L^\mu_Q \lVert \mu - \mu'\rVert_1
		\end{align*}
		where $L^\mu_Q := \rho (L^\mu_R + L^\mu_P/(1-\rho))$. And thus
		\begin{align*}
			\lVert Q^*_{t} - Q^*_{t+1} \rVert_\infty \leq L^\mu_Q \lVert \mu_t - \mu_{t+1} \rVert_1
		\end{align*}
		now that we have shown the Lipschitzness of $Q^*$ with respect to $\mu$. Furthermore, as $\lVert \mu_t - \mu_{t+1} \rVert_1 \leq  c_{\mu,t}$, where $c_{\mu,t} := \frac{c_\mu}{(t+1)^\zeta}$, we get
		\begin{align*}
			\big\lVert \tilde B_{\tau-1,t} (Q^*_{\tau} - Q^*_{t+1}) + \sum_{k=\tau}^t B_{k,t}(Q^*_{k} - Q^*_{t+1}) \big\rVert_\infty \leq L^\mu_Q  \Big[ \tilde \beta_{\tau-1,t} \sum_{l=\tau}^t c_{\mu,l} + \sum_{k=\tau}^t \beta_{k,t} \sum_{l=k}^t c_{\mu,l} \Big]
		\end{align*}
		This concludes the proof.
	\end{proof}
	We next start by bounding the terms $\epsilon_t$ and $\phi_t$ in the error decomposition \eqref{eq:err_decomp}.
	\begin{lemma}
		\begin{align*}
			\lVert \epsilon_t \rVert_\infty & \leq \frac{2}{1-\rho} + C + \bar{w} =: \bar{\epsilon},\\
			\lVert \phi_t\rVert_\infty & \leq (L^\mu_R +  \frac{L^\mu_P}{1-\rho}) \sum_{k=1}^\tau  \lVert \mu_{t-k+1} - \mu_{t-k} \rVert_1 + 2 \bar{\epsilon} \sum_{k=1}^\tau \beta_{t-k}
		\end{align*}
	\end{lemma}
	\begin{proof}
		Recalling the definition of $\epsilon_t$,
		\begin{align*}
			\lVert \epsilon_t \rVert_\infty & = \lVert (e^T_{i_t} e_{i_t} - D_t)(F(\mu_{t-\tau},Q_{t-\tau}) - Q_{t-\tau}) + \beta_t w(t,\mu_t) e_{i_t} \rVert_\infty, \\
			& \leq \lVert e^T_{i_t} e_{i_t} - D_t \rVert_\infty \lVert F(\mu_{t-\tau},Q_{t-\tau}) - Q_{t-\tau} \rVert_\infty + \lvert w(t,\mu_t) \rvert_\infty \lVert  e_{i_t} \rVert_\infty, \\
			& \leq \lVert F(\mu_{t-\tau},Q_{t-\tau}) \rVert_\infty + \lVert Q_{t-\tau} \rVert_\infty + \bar{w} \leq \frac{2}{1-\rho} + C + \bar{w} =: \bar{\epsilon}
		\end{align*}
		where $C = 1$ and $\bar{w} = \frac{2}{1-\rho}$ due to $\lVert Q_t \rVert_\infty \leq \frac{1}{1-\rho}$, contractive property of $F$ and the definitions of noise $w$ \eqref{eq:noise_Q_def} and $Q$-update \eqref{eq:Q_update}. Similarly for $\phi$ we get
		\begin{align}
			\lVert \phi_t \rVert_\infty & = \lVert (e^T_{i_t} e_{i_t} - D_t)(F(\mu_t,Q_t) - F(\mu_{t-\tau},Q_{t-\tau}) - Q_t + Q_{t-\tau}) \rVert_\infty, \nonumber \\
			& \leq \lVert F(\mu_t,Q_t) - F(\mu_{t-\tau},Q_{t-\tau}) \rVert_\infty + \lVert Q_{t-\tau} - Q_t \rVert_\infty, \nonumber\\
			& \leq \sum_{k=1}^\tau \lVert F(\mu_{t-k+1},Q_{t-k+1}) - F(\mu_{t-k},Q_{t-k}) \rVert_\infty + \sum_{k=1}^\tau \lVert Q_{t-k+1} - Q_{t-k} \rVert_\infty. \label{eq:norm_phi_t}
		\end{align}
		We first analyze the first summand in equation \eqref{eq:norm_phi_t}
		\begin{align}
			& \lVert F(\mu_{t-k+1},Q_{t-k+1}) - F(\mu_{t-k},Q_{t-k}) \rVert_\infty \nonumber\\
			& \leq \lVert F(\mu_{t-k+1},Q_{t-k+1}) - F(\mu_{t-k+1},Q_{t-k}) \rVert_\infty + \lVert F(\mu_{t-k+1},Q_{t-k}) - F(\mu_{t-k},Q_{t-k}) \rVert_\infty, \nonumber\\
			& \leq \rho \lVert Q_{t-k+1} - Q_{t-k} \rVert_\infty + \max_{s,a} \lvert R(s,a,\mu_{t-k+1}) - R(s,a,\mu_{t-k}) \rvert \nonumber\\
			& \hspace{5.7cm} + \max_{s,a} \lvert P(\cdot \mid s,a,\mu_{t-k+1}) - P(\cdot \mid s,a,\mu_{t-k}) \rvert \frac{1}{1-\rho}, \nonumber\\
			& \leq  \rho \lVert Q_{t-k+1} - Q_{t-k} \rVert_\infty  + \Big(L^\mu_R +  \frac{L^\mu_P}{1-\rho} \Big) \lVert \mu_{t-k+1} - \mu_{t-k} \rVert_1 \label{eq:first_sumnd_norm_phi}
		\end{align}
		Similarly the second summand in \eqref{eq:norm_phi_t} is
		\begin{align}
			\lVert Q_{t-k+1} - Q_{t-k} \rVert_\infty & = \beta_{t-k} \lVert e^T_{i_{t-k}} \big( F(\mu_{t-k},Q_{t-k}) - Q_{t-k} + w(t-k,\mu_{t-k}) \big) e_{i_{t-k}} \rVert_\infty, \nonumber \\
			& \leq \beta_{t-k} \lVert F(\mu_{t-k},Q_{t-k}) \rVert_\infty + \beta_{t-k} \lVert Q_{t-k} \rVert_\infty + \bar{w} \leq \beta_{t-k} \bar{\epsilon}. \label{eq:second_sumnd_norm_phi}
		\end{align}
		Substituting \eqref{eq:first_sumnd_norm_phi}, \eqref{eq:second_sumnd_norm_phi} into \eqref{eq:norm_phi_t}
		\begin{align*}
			\lVert \phi_t \rVert_\infty & \leq (L^\mu_R + \frac{L^\mu_P}{1-\rho} ) \sum_{k=1}^\tau  \lVert \mu_{t-k+1} - \mu_{t-k} \rVert_1 + 2 \bar{\epsilon} \sum_{k=1}^\tau \beta_{t-k}.
		\end{align*}
	\end{proof}
	\noindent Having proved bounds on $\epsilon_t$ and $\phi_t$, we now prove some properties of the learning rates $c_{\mu,t} := \frac{c_\mu}{(t+1)^\zeta}$ and $\beta_t = \frac{c_\beta}{(t+1)^\nu}$ where $0.5 < \nu \leq 1$, $\zeta > 1$ and $c_\beta \geq \frac{\nu}{\sigma}$.
	\begin{lemma} \label{lem:learn_rate_prop}
		Below we present some results regarding the learning rate $\beta_t$ and the associated variables.
		\begin{enumerate}
			\item $
			\tilde \beta_{k,t} \leq \big( \frac{k+2}{t+2} \big)^{c_\beta \sigma} \leq \big( \frac{k+2}{t+2} \big)^{\nu},
			$
			\item $
			\beta_{k,t} \leq \frac{c_\beta}{(k+1)^\nu} \big( \frac{k+2}{t+2} \big)^{c_\beta \sigma} \leq 2 \frac{c_\beta}{(t+2)^\nu},
			$
			\item $
			\sum_{k=1}^t \beta^2_{k,t} \leq \frac{2 c_\beta^2}{2 c_\beta \sigma - 2 \nu +1} \frac{1}{(t+2)^{2\nu -1}},
			$
			\item $
			\sum_{k=\tau}^t \beta_{k,t} \sum_{l = k -\tau}^{k-1} \beta_l \leq \frac{2 c_\beta^2 (\tau + 1)^\nu \tau}{1 + c_\beta \sigma - 2\nu} \frac{1}{(t+2)^{2\nu - 1}}
			$
		\end{enumerate}
	\end{lemma}
	\begin{proof}
		For part (1) we start by recalling the definition of $\tilde{\beta}_{k,t}$ for $k \in [t]$
		\begin{align*}
			\tilde{\beta}_{k,t} & = \prod_{l=k+1}^t (1 - \beta_l \sigma) \leq \prod_{l=k+1}^t 1 - \frac{c_\beta}{l+1}  = \prod_{l=k+1}^t e^{\log (1 - \frac{c_\beta}{l+1})}, \\
			& \leq \prod_{l=k+1}^t e^{ - \frac{c_\beta}{l+1}} = e^{-\sum_{l = k+1}^t \frac{c_\beta}{l+1}} = \exp \Bigg( -\sum_{l = k+1}^t \frac{c_\beta}{l+1} \Bigg), \\
			& \leq \exp \Bigg( - \int_{k+1}^{t+1} \frac{c_\beta \sigma}{y+1} dy \Bigg) = \exp \Bigg( -c_\beta \sigma \log \Bigg( \frac{t+2}{k+2} \Bigg) \Bigg), \\
			& = \Bigg( \frac{k+2}{t+2} \Bigg)^{c_\beta \sigma} \leq  \Bigg( \frac{k+2}{t+2} \Bigg)^{\nu}.
		\end{align*}
		The first inequality is due to the fact that $\beta_t := \frac{c_\beta}{(t+1)^\nu} > \frac{c_\beta}{t+1}$ since $\nu < 1$. The last inequality is due to the fact that $c_\beta \sigma \geq \nu$ and $k \leq t$.
		
		\noindent For part (2), recalling the definition of $\beta_{k,t}$ and using the bound on $\tilde \beta_{k,t}$, we get
		\begin{align*}
			\beta_{k,t} = \beta_k \tilde{\beta}_{k,t} \leq \frac{c_\beta}{(k+1)^\nu} \Bigg( \frac{k+2}{t+2} \Bigg)^{c_\beta \sigma} \leq 2 \frac{c_\beta}{(t+2)^\nu}.
		\end{align*}
		
		\noindent For part (3), 
		analyzing each summand
		\begin{align*}
			\beta^2_{k,t} & \leq \frac{c^2_\beta}{(k+1)^{2 \nu}} \Bigg( \frac{k+2}{t+2} \Bigg)^{2 c_\beta \sigma} =\frac{c^2_\beta}{(t+2)^{2 c_\beta \sigma}} \frac{(k+2)^{2 c_\beta \sigma}}{(k+1)^{2 \nu}}, \\
			& \leq \frac{2 c^2_\beta}{(t+2)^{2 c_\beta \sigma}} (k+1)^{2 c_\beta \sigma - 2 \nu}
		\end{align*}
		Substituting into the sum
		\begin{align*}
			\sum_{k=1}^t \beta^2_{k,t} & \leq \sum_{k=1}^t \frac{2 c^2_\beta}{(t+2)^{2 c_\beta \sigma}} (k+1)^{2 c_\beta \sigma - 2 \nu} \leq \frac{2 c^2_\beta}{(t+2)^{2 c_\beta \sigma}} \int_{1}^{t+1} (y+1)^{2 c_\beta \sigma - 2 \nu} dy, \\
			& \leq \frac{2 c^2_\beta}{(t+2)^{2 c_\beta \sigma}} \frac{(t+2)^{2 c_\beta \sigma - 2 \nu + 1}}{2 c_\beta \sigma - 2 \nu + 1} = \frac{2 c_\beta^2}{2 c_\beta \sigma - 2 \nu +1} \frac{1}{(t+2)^{2\nu -1}}
		\end{align*}
		\noindent For part (4), as $k - \tau \leq l \leq k-1$, in the expression $\sum_{k=\tau}^t \beta_{k,t} \sum_{l = k -\tau}^{k-1} \beta_l$, we get
		\begin{align*}
			\beta_l = \frac{c_\beta}{(l+1)^\nu} \leq \frac{c_\beta}{(k-\tau+1)^\nu} \leq \frac{c_\beta (\tau+1)^\nu}{(k+1)^\nu}
		\end{align*}
		The summation can be written down as
		\begin{align*}
			\sum_{k=\tau}^t \beta_{k,t} \sum_{l = k -\tau}^{k-1} \beta_l & \leq \sum_{k=\tau}^t \beta_{k,t} \frac{c_\beta \tau(\tau+1)^\nu}{(k+1)^\nu} \leq \sum_{k=\tau}^t \frac{c_\beta \tau}{(k+1)^\nu} \Bigg( \frac{k+2}{t+2} \Bigg)^{c_\beta \sigma} \frac{c_\beta (\tau + 1)^\nu}{(k+1)^\nu}, \\
			& \leq \sum_{k=\tau}^t \frac{2 c^2_\beta (\tau + 1)^\nu \tau}{(t+2)^{c_\beta \sigma}} (k+1)^{c_\beta \sigma - 2 \nu} \leq \frac{2 c^2_\beta (\tau + 1)^\nu \tau}{(t+2)^{c_\beta \sigma}} \int_{\tau}^{t+1} (y+1)^{2\nu - c_\beta \sigma} dy, \\
			& \leq \frac{2 c^2_\beta (\tau + 1)^\nu \tau}{(t+2)^{c_\beta \sigma}} \frac{(t+2)^{1+c_\beta \sigma - 2 \nu}}{1 + c_\beta \sigma - 2 \nu} \leq \frac{2 c_\beta^2 (\tau + 1)^\nu \tau}{1 + c_\beta \sigma - 2\nu} \frac{1}{(t+2)^{2\nu - 1}}.
		\end{align*}
		Hence the inequalities have been proved.
	\end{proof}
	Having proved some properties of the learning rates in Lemma \ref{lem:learn_rate_prop} we are now able to bound the two parts of the quantity $\Big\lVert \sum_{k=\tau}^t \beta_k \tilde B_{k,t}(\epsilon_k + \phi_k) \Big\rVert_\infty $ as follows. The bound on the first quantity relies on the properties of the learning rates and the second bound relies on the fact that $\epsilon_t$ is a Martingale Difference sequence.
	\begin{lemma} \label{lem:phi_eps_bound}
		\begin{align*}
			\Big\lVert \sum_{k=\tau}^t \beta_k \tilde B_{k,t} \phi_k \Big\rVert_\infty & \leq  \frac{C^1_\phi}{(t+2)^{2 \nu - 1}} + \frac{C^2_\phi}{(t+2)^{\zeta+\nu-1}}, \\
			\Big\lVert \sum_{k=\tau}^t \beta_k \tilde B_{k,t}\epsilon_k \Big\rVert_\infty & \leq \frac{C_\epsilon}{(t+2)^{\nu - 1/2}}
		\end{align*}
		with probability at least $1-\delta_Q$, where
		\begin{align}
			C^1_\phi & = \frac{4 c^2_\beta (1+\tau)^\nu \tau}{1 + c_\beta \sigma - 2 \nu} \bar{\epsilon}, \\
			C^2_\phi & = \Big(L^\mu_R +  \frac{L^\mu_P}{1-\rho} \Big) \frac{2 c_\mu c_\beta \tau (1+\tau)^\zeta}{c_\beta \sigma - \nu - \zeta +1},\\
			C_\epsilon &= \frac{10 \bar{\epsilon}}{\sqrt{2 c_\beta \sigma - 2 \nu + 1}} \sqrt{(\tau + 1)c_\beta^2 \log\Big( \frac{2(\tau+1)T^2 SA}{\delta_Q} \Big)}. \label{eq:C_phi_eps}
		\end{align}
	\end{lemma}
	\begin{proof}
		We start with the first inequality
		\begin{align}
			\Big\lVert \sum_{k=\tau}^t \beta_k \tilde B_{k,t} \phi_k \Big\rVert_\infty & \leq \sum_{k=\tau}^t \beta_{k,t} \lVert \phi_k \rVert_\infty \nonumber \\
			& \leq \sum_{k=\tau}^t  \beta_{k,t} \Big( \Big(L^\mu_R +  \frac{L^\mu_P}{1-\rho} \Big) \sum_{l=1}^\tau  \lVert \mu_{k-l+1} - \mu_{k-l} \rVert_\infty + 2 \bar{\epsilon} \sum_{l=1}^\tau \beta_{k-l} \Big), \nonumber \\
			& \leq 2 \bar{\epsilon} \sum_{k=\tau}^t \beta_{k,t} \sum_{l=k-\tau}^{k-1} \beta_l + \Big(L^\mu_R +  \frac{L^\mu_P}{1-\rho} \Big) \sum_{k=tau}^t \beta_{k,t} \sum_{l=k-\tau}^{k-1} \lVert \mu_{l+1} - \mu_l \rVert_\infty, \nonumber \\
			& \leq C^1_\phi \frac{1}{(t+2)^{2\nu - 1}} + \Big(L^\mu_R +  \frac{L^\mu_P}{1-\rho} \Big) \sum_{k=tau}^t \beta_{k,t} \sum_{l=k-\tau}^{k-1}  c_{\mu,l}, \label{eq:B_phi_inter_bound}
		\end{align}
		Now we obtain an upper bound for the expression $\sum_{k=\tau}^t \beta_{k,t} \sum_{l=k-\tau}^{k-1}  c_{\mu,l}$. Due to the fact that $k-\tau \leq l \leq k-1$ in the expression $\sum_{k=tau}^t \beta_{k,t} \sum_{l=k-\tau}^{k-1}  c_{\mu,t} $ and thus $c_{\mu,l} = \frac{c_\mu}{(l+1)^\zeta} \leq \frac{c_\mu (\tau+1)^\zeta}{(k+1)^\zeta}$
		\begin{align}
			\sum_{k=tau}^t \beta_{k,t} \sum_{l=k-\tau}^{k-1}  c_{\mu,l} & \leq \sum_{k=\tau}^t \beta_{k,t} \frac{c_\mu (\tau+1)^\zeta \tau}{(k+1)^\zeta} \leq \sum_{k=\tau}^t \frac{c_\mu \tau}{(k+1)^\zeta} \Bigg( \frac{k+2}{t+2} \Bigg)^{c_\beta \sigma} \frac{c_\beta (\tau+1)^\zeta}{(k+1)^\nu}, \nonumber \\
			& \leq 2 \sum_{k-\tau}^t \frac{c_\mu c_\beta \tau (\tau+1)^\zeta}{(t+2)^{c_\beta \sigma}} \frac{1}{(k+1)^{\nu + \zeta - c_\beta \sigma}}, \nonumber \\
			& \leq \frac{c_\mu c_\beta \tau (\tau+1)^\zeta}{(t+2)^{c_\beta \sigma}} \frac{(t+2)^{c_\beta \sigma - \nu - \zeta +1}}{c_\beta \sigma - \nu - \zeta +1} \leq \frac{c_\mu c_\beta \tau (\tau+1)^\zeta}{c_\beta \sigma - \nu - \zeta +1} \frac{1}{(t+2)^{ \nu + \zeta -1}}, \label{eq:B_c_dbl_sum}
		\end{align}
		since $c_\beta \sigma \geq \nu + \zeta -1$. Substituting \eqref{eq:B_c_dbl_sum} into \eqref{eq:B_phi_inter_bound} we get
		\begin{align*}
			\Big\lVert \sum_{k=\tau}^t \beta_k \tilde B_{k,t} \phi_k \Big\rVert_\infty & \leq  \frac{C^1_\phi}{(t+2)^{2 \nu - 1}} + \frac{C^2_\phi}{(t+2)^{\zeta+\nu-1}}.
		\end{align*}
		
		Next we move to the second inequality. Recalling the definition of $\epsilon_t$
		\begin{align*}
			\epsilon_t = (e^T_{i_t} e_{i_t} - D_t)(F(\mu_{t-\tau},Q_{t-\tau}) - Q_{t-\tau}) + \beta_t w(t,\mu_t) e_{i_t}
		\end{align*}
		which is $\Fs_{t+1}$ shifted martingale difference sequence, $\EE[\epsilon_t \mid \Fs_t-\tau] = 0$. We will use a variant of the Azuma-Hoeffding bound which can handle \emph{shifted} Martingale Difference Sequences \cite{qu2020finite}. Each element in the vector $\sum_{k=\tau}^t \beta_k \tilde B_{k,t}\epsilon_k$ can be upper bounded by $\lvert \sum_{k=\tau}^t \beta_k \epsilon_{i,k} \tilde b_{k,t,i}\rvert$ where $\epsilon_{i,k}$ is the $i$th element in the vector $\epsilon_k$. Using Lemmas 13 \& 14 from \cite{qu2020finite} we get
		\begin{align*}
			\bigg	\lvert \sum_{k=\tau}^t \beta_k \epsilon_{i,k} \tilde b_{k,t,i} \bigg\rvert & = \bigg\lvert \sum_{k=\tau}^t \beta_k \epsilon_{i,k} \prod_{l=k+1}^t (1 - \beta_l d_{l,i}) \bigg\rvert \leq \sup_{\tau \leq k_0 < t} \bigg( \bigg\lvert \sum_{k = k_0 + 1}^t \beta_{k,t} \epsilon_{i,k}  \bigg \rvert + 2 \bar{\epsilon} \beta_{k_0,t} \bigg), \\
			& \leq \bar{\epsilon} \sqrt{2 (\tau + 1) \sum_{k=\tau+1}^t \beta^2_{k,t} \log \bigg( \frac{2 (\tau+1)tSA}{\delta_Q} \bigg)} + \sup_{\tau \leq k_0 \leq t} 2 \bar{\epsilon} \beta_{k_0,t}, \\
			& \leq \frac{2\bar{\epsilon}}{\sqrt{2 c_\beta \sigma - 2 \nu +1}} \sqrt{ \frac{(\tau + 1) c_\beta^2}{(t+2)^{2\nu -1}}  \log \bigg( \frac{2 (\tau+1)tSA}{\delta_Q} \bigg)} \\
			& \hspace{5.3cm} + \sup_{\tau \leq k_0 \leq t} 2 \bar{\epsilon} \frac{c_\beta}{(k_0 + 1)^\nu} \Bigg( \frac{k_0 + 2}{t+2} \Bigg)^{c_\beta \sigma}, \\
			& \leq \frac{2\bar{\epsilon}}{\sqrt{2 c_\beta \sigma - 2 \nu +1}} \sqrt{ \frac{(\tau + 1) c_\beta^2}{(t+2)^{2\nu -1}}  \log \bigg( \frac{2 (\tau+1)tSA}{\delta_Q} \bigg)} + 4 \bar{\epsilon} \frac{c_\beta}{(t+2)^\nu}, \\
			& \leq \frac{10 \bar{\epsilon}}{\sqrt{2 c_\beta \sigma - 2 \nu +1}} \sqrt{ \frac{(\tau + 1) c_\beta^2}{(t+2)^{2\nu -1}}  \log \bigg( \frac{2 (\tau+1)tSA}{\delta_Q} \bigg)}.
		\end{align*}
		with probability at least $1-\delta_Q/SA$. Applying the union bound over $\forall i \in \Ss \times \As$, we get
		\begin{align*}
			\Big\lVert \sum_{k=\tau}^t \beta_k \tilde B_{k,t}\epsilon_k \Big\rVert_\infty & \leq \frac{C_\epsilon}{(t+2)^{\nu - 1/2}}, \hspace{0.2cm} C_\epsilon = \frac{10 \bar{\epsilon}}{\sqrt{2 c_\beta \sigma - 2 \nu + 1}} \sqrt{(\tau + 1)c_\beta^2 \log\Big( \frac{2(\tau+1)t SA}{\delta} \Big)}
		\end{align*}
	with probability at least $1-\delta_Q$.
	\end{proof}
	Now we aim to bound the last term in \eqref{eq:err_decomp}
	\begin{lemma} \label{lem:c_mu_bound}
		If $\zeta > 1$, then
		\begin{align*}
			\tilde \beta_{\tau-1,t} \sum_{l=\tau}^t c_{\mu,l} + \sum_{k=\tau}^t \beta_{k,t} \sum_{l=k}^t c_{\mu,l} \leq\frac{c_\mu}{\zeta - 1} \frac{1}{\tau^{\zeta -1}}
		\end{align*}
	\end{lemma}
\textit{Proof.}
		\begin{align*}
			\tilde \beta_{\tau-1,t} \sum_{l=\tau}^t c_{\mu,l} + \sum_{k=\tau}^t \beta_{k,t} \sum_{l=k}^t c_{\mu,l} & \leq \tilde \beta_{\tau-1,t} \sum_{l=\tau}^t c_{\mu,l} + \sum_{k=\tau}^t \beta_{k,t} \sum_{l=\tau}^t c_{\mu,l} \\ 
			& \leq \sum_{l=\tau}^t c_{\mu,l} \\
			& \leq \frac{c_\mu}{\zeta - 1} \frac{1}{\tau^{\zeta -1}}. \hspace{6.4cm} \qed
		\end{align*}  

	\noindent Now that we have bounded all the terms in \eqref{eq:err_decomp} we will show that the error term $e^Q_t$ can be bounded by a decreasing function of time $t$. Toward this end we introduce a lemma that will help us with the proof of the main result.
	\begin{lemma} \label{lem:helper_lemma}
		For any $0 < w < 1$ and $t \geq \tau$,
		\begin{align*}
			e_t := \sum_{k=\tau}^t b_{k,t,i} \frac{1}{(k+1)^w} \leq \frac{1}{\sqrt{\rho}(t+2)^w}, \hspace{0.5cm} g_t:=\sum_{k=\tau}^t b_{k,t,i} \frac{1}{\tau^{\zeta-1}} \leq \frac{1}{\sqrt{\rho} \tau^{\zeta-1}}
		\end{align*}
	\end{lemma}
	\begin{proof}
		Recall that $b_{k,t,i} = \beta_k d_{k,i} \prod_{l=k+1}^t (1 - \beta_l d_{l,i})$. We first prove the inequality for $e_t$ by recursion. We start with the base case.
		\begin{align*}
			e_\tau & = b_{\tau,\tau,i} \frac{1}{(\tau + 1)^w} = \beta_\tau d_{\tau,i} \frac{1}{(\tau+1)^w}, \\
			& = \beta_\tau d_{\tau,i} \Bigg( \frac{\tau+2}{\tau+1} \Bigg)^w \frac{1}{(\tau+2)^w} = \beta_\tau d_{\tau,i} \Bigg( 1 + \frac{1}{\tau+1} \Bigg)^w \frac{1}{(\tau+2)^w}
		\end{align*}
		and since $\tau$ is chosen such that $\Big(1 + \frac{1}{\tau+1}\Big)^w \leq \frac{1}{\sqrt\rho}$ for $w \leq 1$, we have
		\begin{align*}
			e_\tau \leq \frac{1}{\sqrt{\rho}(\tau + 2)^w}.
		\end{align*}
		Now assume that for some $t > \tau$, $e_{t-1} \leq \frac{1}{\sqrt{\rho}(t+1)^w}$. Then
		\begin{align*}
			e_t & = \sum_{k=\tau}^{t-1} b_{k,t,i} \frac{1}{(k+1)^w} + b_{t,t,i} \frac{1}{(t+1)^w}, \\
			& = (1 - \beta_t d_{t,i}) \sum_{k=\tau}^{t-1} b_{k,t-1,i} \frac{1}{(k+1)^w} + \beta_t d_{t,i} \frac{1}{(t+1)^w}, \\
			& = (1-\beta_t d_{t,i}) e_{t-1} + \beta_t d_{t,i} \frac{1}{(t+1)^w}, \\
			& \leq (1-\beta_t d_{t,i}) \frac{1}{\sqrt{\rho}(t+1)^w} + \beta_t d_{t,i} \frac{1}{(t+1)^w}, \\
			& = \frac{1 - \beta_t d_{t,i}(1- \sqrt{\rho})}{\sqrt{\rho}(t+1)^w}, \\
			& \leq \Big( 1 - \frac{c_\beta \sigma}{(t+1)^\nu} (1-\sqrt{\rho}) \Big) \frac{1}{\sqrt{\rho}(t+1)^w}, \\
			& = \Big( 1 - \frac{c_\beta \sigma}{(t+1)^\nu} (1-\sqrt{\rho}) \Big) \frac{(t+2)^w}{(t+1)^w} \frac{1}{\sqrt{\rho}(t+2)^w}, \\
			& = \Big( 1 - \frac{c_\beta \sigma}{(t+1)^\nu} (1-\sqrt{\rho}) \Big) \Big( 1+\frac{1}{t+1} \Big)^w \frac{1}{\sqrt{\rho}(t+2)^w}.
		\end{align*}
		For any $x > -1$, $(1+x) \leq e^x$ and thus
		\begin{align*}
			\Big( 1 - \frac{c_\beta \sigma}{(t+1)^\nu} (1-\sqrt{\rho}) \Big) \Big( 1+\frac{1}{t+1} \Big)^w \leq e^{- \frac{c_\beta \sigma}{(t+1)^\nu} (1 - \sqrt{\rho}) + \frac{w}{t+1}} \leq 1,
		\end{align*}
		where the last inequality is due to $c_\beta \geq \frac{1}{(1-\sqrt\rho) \sigma}$. Hence we have proved that
		\begin{align*}
			e_t \leq \frac{1}{\sqrt{\rho}(t+2)^w}.
		\end{align*}
		Now we prove the inequality for $g_t$ using recursion again. For the base case it is easy to see that
		\begin{align*}
			g_\tau = b_{\tau,\tau,i} \frac{1}{\tau^{\zeta-1}} = \beta_\tau d_{\tau,i} \frac{1}{\tau^{\zeta-1}} \leq \frac{1}{\sqrt{\rho}\tau^{\zeta-1}}.
		\end{align*}
		Now assume that $g_{t-1} \leq \frac{1}{\sqrt{\rho}\tau^{\zeta-1}}$. Then
		\begin{align*}
			g_t & = (1-\beta_t d_{t,i}) g_{t-1} + \beta_t d_{t,i} \frac{1}{\tau^{\zeta-1}} \leq (1-\beta_t d_{t,i}) \frac{1}{\sqrt{\rho}\tau^{\zeta-1}} + \beta_t d_{t,i} \frac{1}{\tau^{\zeta-1}}, \\
			& \leq \frac{1 - \beta_t d_{t,i}(1 - \sqrt\rho)}{\sqrt\rho \tau^{\zeta-1}} \leq \Big(1 - \frac{c_\beta \sigma}{(t+1)^\nu} (1-\sqrt\rho)\Big) \frac{1}{\sqrt\rho \tau^{\zeta-1}} \leq \frac{1}{\sqrt\rho \tau^{\zeta-1}},
		\end{align*}
		which proves the recursion step and completes the proof.
	\end{proof}
	\noindent Now we prove the main result using Lemma \ref{lem:helper_lemma}. Recalling \eqref{eq:err_decomp},
	\begin{align}
		e^Q_{t+1} & \leq \tilde B_{\tau-1,t} e^Q_{\tau} + \rho \sup_{i} \sum_{k = \tau}^t b_{k,t,i} e^Q_k + \Big\lVert \sum_{k=\tau}^t \beta_k \tilde B_{k,t} (\epsilon_k + \phi_k) \Big\rVert_\infty \nonumber\\
		& \hspace{6cm} +  L^\mu_Q \Big[ \tilde \beta_{\tau-1,t} \sum_{l=\tau}^t c_{\mu,l} + \sum_{k=\tau}^t \beta_{k,t} \sum_{l=k}^t c_{\mu,l} \Big]
	\end{align}
	Using Lemmas \ref{lem:phi_eps_bound} and \ref{lem:c_mu_bound} and using $C_\mu := 10 L^\mu_Q  c_\mu$ for $\zeta \geq 1.1$,
	\begin{align}
		e^Q_{t+1} & \leq \tilde B_{\tau-1,t} e^Q_{\tau} + \rho \sup_{i} \sum_{k = \tau}^t b_{k,t,i} e^Q_k + \frac{C^1_\phi}{(t+2)^{2 \nu - 1}} + \frac{C^2_\phi}{(t+2)^{\zeta+\nu-1}} + \frac{C_\epsilon}{(t+2)^{\nu - 1/2}} +  \frac{C_\mu}{\tau^{\zeta -1}}
	\end{align}
	We will prove that $e^Q_t \leq \frac{\bar{C}_1}{(t+1)^{2\nu - 1}} + \frac{\bar{C}_2}{(t+1)^{\zeta + \nu - 1}} + \frac{\bar{C}_3}{(t+1)^{\nu - 1/2}} + \frac{\bar{C}_4}{\tau^{\zeta - 1}}$ using induction. The base case is trivially true; now assume this to be true for $t$:
	\begin{align*}
		e^Q_{t+1} & \leq \tilde B_{\tau-1,t} e^Q_{\tau} + \rho \sup_{i} \sum_{k = \tau}^t b_{k,t,i} \bigg( \frac{\bar{C}_1}{(t+1)^{2\nu - 1}} + \frac{\bar{C}_2}{(t+1)^{\zeta + \nu - 1}} + \frac{\bar{C}_3}{(t+1)^{\nu - 1/2}} + \frac{\bar{C}_4}{\tau^{\zeta - 1}} \bigg) \\
		& \hspace{5cm} + \frac{C^1_\phi}{(t+2)^{2 \nu - 1}} + \frac{C^2_\phi}{(t+2)^{\zeta+\nu-1}} + \frac{C_\epsilon}{(t+2)^{\nu - 1/2}} +  \frac{C_\mu}{\tau^{\zeta -1}}, \\
		& \leq \frac{\sqrt\rho \bar{C}_1 + C^1_\phi}{(t+2)^{2 \nu - 1}} + \frac{\sqrt\rho \bar{C}_2 + C^2_\phi}{(t+2)^{\zeta+\nu-1}} + \frac{\sqrt\rho \bar{C}_3 + C_\epsilon + 2 (\tau+1)^\nu/(1-\rho)}{(t+2)^{\nu - 1/2}} +  \frac{\sqrt\rho \bar{C}_4 + C_\mu}{\tau^{\zeta -1}}, \\
		& \leq \frac{\bar{C}_1}{(t+2)^{2 \nu - 1}} + \frac{\bar{C}_2}{(t+2)^{\zeta+\nu-1}} + \frac{\bar{C}_3 }{(t+2)^{\nu - 1/2}} +  \frac{\bar{C}_4 }{\tau^{\zeta -1}}
	\end{align*}
	with probability $1-\delta_Q$ (using a union bound type argument) where $$\bar{C}_1 = \frac{C^1_\phi}{1 - \sqrt \rho}, \bar{C}_2 = \frac{C^2_\phi}{1 - \sqrt \rho}, \bar{C}_3 = \frac{C_\epsilon + 2 (\tau+1)^\nu/(1-\rho)}{1 - \sqrt \rho}, \bar{C}_4 = \frac{C_\mu}{1 - \sqrt \rho},$$Finally
	\begin{align*}
		\epsilon_Q & = \lVert Q_T - Q^*_1 \rVert_\infty, \\
		& \leq \lVert Q_T - Q^*_T \rVert_\infty + \lVert Q^*_T - Q^*_1 \rVert_\infty, \\
		& \leq  e^Q_T + \sum_{t=1}^{T-1} \lVert Q^*_{t+1} - Q^*_t \rVert_\infty, \\
		& \leq \frac{\bar{C}_1}{(t+2)^{2 \nu - 1}} + \frac{\bar{C}_2}{(t+2)^{\zeta+\nu-1}} + \frac{\bar{C}_3 }{(t+2)^{\nu - 1/2}} +  \frac{\bar{C}_4 }{\tau^{\zeta -1}} + L^\mu_Q \sum_{k=1}^{T-1} \lVert \mu_{t+1} - \mu_t \rVert_1, \\
		& \leq \frac{\bar{C}_1}{(t+2)^{2 \nu - 1}} + \frac{\bar{C}_2}{(t+2)^{\zeta+\nu-1}} + \frac{\bar{C}_3 }{(t+2)^{\nu - 1/2}} +  \frac{\bar{C}_4 }{\tau^{\zeta -1}} + L^\mu_Q \sum_{k=1}^{T-1} c_{\mu,t}, \\
		& = \Os(T^{1-2\nu}) + \Os(T^{1-\zeta-\nu}) + \tilde \Os(T^{1/2-\nu}) + \Os(2^{1-\zeta})
	\end{align*}
given that $\zeta \geq 1.1$ with probability at least $1-\delta_Q$.
\end{proof}

\subsection{Proof of Theorem \ref{thm:conv_bound}}
\begin{proof}

In this proof we provide finite sample bounds for the convergence of approximation errors in control policy and mean-field, $e^k_\pi$ and $e^k_\mu$, respectively. We start by characterizing the approximation errors in control policy and mean field $e^k_\pi$ and $e^k_\mu$ on the first timestep in each episode $k$. Then the evolutions of these approximation errors are studied under two timescale learning rates. First we analyze the approximation error in control policy $e^k_\pi$ which is evolving at a faster learning rate compared to the approximation error in the mean-field $e^k_\mu$. This error is shown to converge due to the good approximation of the $Q$-function (Lemma \ref{lem:Q_learn_conv}), increase of Lipschitz coefficient $\lambda^k$ at a logarithmic rate and fast learning rate $c^k_\pi$. Next the approximation error in mean-field $e^k_\mu$ (which is evolving under the slower timescale) is also shown to converge due to the good transition dynamics estimation (Lemma \ref{lem:tran_prob_conv}), the contraction mapping property (Assumption \ref{asm:contrct}) and the convergence of $e^k_\pi$.

	First we recall the update rules in Algorithm \ref{alg:Single_loop_RL}
	\azedit{
	\begin{align*}
		\mu^k_t & = \PP_{S(\epsilon^\net)} \big[(1-c^k_{\mu,t}) \mu^k_{t-1} + c^k_{\mu,t} \hGamma^k_{1,t}, \bbone_{t = 1} \big], \text{ where } \hGamma^k_{1,t} = (\hat{P}^k_{t})^\top \mu^k_{t-1} \\
		\pi^k_{t} & = (1-c^k_{\pi,t}) \pi^k_{t-1} + c^k_{\pi,t} \big( (1-\psi)  \hGamma^k_{2,t}  + \psi \mathds{1}_{|\As|} \big),  \text{ where } \hGamma^k_{2,t} = \softmax{\lambda}{\cdot,Q^k_t}
	\end{align*}}
\noindent where $\hGamma^k_{1,t}$ and $\hGamma^k_{2,t}$ are the approximate consistency and optimality operators. The RL update can now be written down for the first timestep of episode $k+1$,
	\begin{align*}
		\mu^{k+1}_1 & = \PP_{S(\epsilon^\net)} \big[(1-c^{k+1}_{\mu,1}) \mu^{k+1}_{0} + c^{k+1}_{\mu,1} (\hat{P}^{k+1}_1)^\top \mu^{k+1}_{0},1\big], \\
		& = \PP_{S(\epsilon^\net)} \big[(1-c^{k+1}_{\mu,1}) \mu^{k}_{T} + c^{k+1}_{\mu,1} (\hat{P}^{k}_{T})^\top \mu^{k}_{T},1 \big], \\	
		& = \PP_{S(\epsilon^\net)} \big[(1-c^{k+1}_{\mu,1}) (\mu^{k}_{1} + \Delta^k_\mu) + c^{k+1}_{\mu,1} (\hat{P}^{k}_{T})^\top (\mu^{k}_{1} + \Delta^k_\mu),1 \big], \\	
		\pi^{k+1}_{1} & = (1-c^{k+1}_{\pi,1}) \pi^{k+1}_{0} + c^{k+1}_{\pi,1} \big((1-\psi)\softmax{\lambda}{\cdot,Q^{k+1}_1} + \psi \bbone_{|\As|} \big), \\  
		& = (1-c^{k+1}_{\pi,1}) \pi^{k}_{T} + c^{k+1}_{\pi,1} \big((1-\psi)\softmax{\lambda}{\cdot,Q^{k+1}_1} + \psi \bbone_{|\As|} \big), \\  
		& = (1-c^{k+1}_{\pi,1}) (\pi^{k}_{1} + \Delta^k_\pi) + c^{k+1}_{\pi,1} \big((1-\psi)\softmax{\lambda}{\cdot,Q^{k+1}_1} + \psi \bbone_{|\As|} \big),
	\end{align*}
\noindent where $\Delta^k_\mu := \mu^{k}_{T} - \mu^{k}_{1}$ and $\Delta^k_\pi := \pi^{k}_{T} - \pi^{k}_{1}$ are the drifts in mean-field and policy, respectively, in the episode $k$. Since all the time indices in the above inequalities are $1$, we suppress all time indices from here on. Coupled with the fact that $c^{k+1}_{\mu,1} = c^{k+1}_{\mu}$ and $c^{k+1}_{\pi,1}= c^{k+1}_{\pi}$, the update rules can be written as
	\azedit{
	\begin{align}
		\mu^{k+1}	& = \PP_{S(\epsilon^\net)} \big[ (1-c^{k+1}_{\mu}) (\mu^{k} + \Delta^k_\mu) + c^{k+1}_{\mu} (\hat{P}^{k})^\top (\mu^{k} + \Delta^k_\mu),1 \big], \nonumber \\	
		\pi^{k+1} & = (1-c^{k+1}_{\pi}) (\pi^{k} + \Delta^k_\pi) + c^{k+1}_{\pi} \big((1-\psi)\softmax{\lambda}{\cdot,Q^{k+1}} + \psi \bbone_{|\As|} \big). \label{eq:sngl_loop_upd_RL}
	\end{align}}
	Here we use $\hat{P}^{k}:= \hat{P}^{k}_{T}$ and $Q^{k}:=Q^{k}_T$ for conciseness. 
	The estimation errors for transition matrix and $Q$-function are denoted as 
	\begin{align*}
		\epsilon^k_P := \lVert \hat{P}^k - P_{\pi^k,\mu^k} \rVert_F, \hspace{0.5cm} \epsilon^k_Q := \lVert Q^k - Q^*_{\mu^k} \rVert_\infty .
	\end{align*}
	Now we compute the evolution of the approximation errors. We start with $e^k_\pi := \lVert \pi^k - \hGamma^\lambda_1(\mu^k) \rVert_{TV}$:
	\begin{align}
		e^{k+1}_\pi & = \lVert \pi^{k+1} - \hGamma^\lambda_1(\mu^{k+1}) \rVert_{TV}, \nonumber \\
		& \leq \lVert  \pi^{k+1} - \hGamma^\lambda_1(\mu^{k}) \rVert_{TV} + \lVert \hGamma^\lambda_1(\mu^{k}) - \hGamma^\lambda_1(\mu^{k+1}) \rVert_{TV}, \nonumber \\
		& \leq \lVert (1-c^{k+1}_\pi) (\pi^k + \Delta^k_\pi) + c^{k+1}_\pi \softmax{\lambda}{\cdot,Q^k} - \softmax{\lambda}{\cdot,Q^*_{\mu^k}} \rVert_{TV} + d_1 \lVert \mu^{k+1} - \mu^k \rVert_1 + 2 c^{k+1}_\pi \psi, \nonumber \\
		& \leq (1-c^{k+1}_\pi) \lVert \pi^k - \hGamma^\lambda_1(\mu^k) \rVert_{TV} + (1-c^{k+1}_\pi) \lVert \Delta^k_\pi \rVert_{TV} \nonumber \\
		& \hspace{3.5cm} + c^{k+1}_\pi \lVert \softmax{\lambda}{\cdot,Q^k} - \softmax{\lambda}{\cdot,Q^*_{\mu^k}}  \rVert_{TV} + d_1 \lVert \mu^{k+1} - \mu^k \rVert_1 + c^{k+1}_\pi \epsilon/2, \nonumber \\
		& \leq (1-c^{k+1}_\pi) e^k_\pi + \lVert \Delta^k_\pi \rVert_{TV} + c^{k+1}_\pi \lVert \softmax{\lambda}{\cdot,Q^k} - \softmax{\lambda}{\cdot,Q^*_{\mu^k}} \rVert_{TV} \nonumber \\
		& \hspace{3.5cm} 
  + d_1 \lVert \mu^{k+1} - \mu^k \rVert_1 + c^{k+1}_\pi \epsilon/2. \label{eq:e_k_pi}
	\end{align}
	where the third inequality is due to $\psi \leq \epsilon/4$. To simplify the above expression we prove the Lipschitz property of 
 the $\softmax{\lambda}{\cdot,Q}$ operator.
		\begin{lemma} \label{lem:softmax_Lip}
		The $\softmax{\lambda}{\cdot,Q}$ satisfies the Lipschitz property for $\lambda > 0$ and $Q:\Ss \times \As \rightarrow \RR^+$,
		\begin{align*}
			\lVert \softmax{\lambda}{\cdot,Q} - \softmax{\lambda}{\cdot,Q'} \rVert_{TV} & \leq \lambda S \sqrt{A} \lVert Q - Q' \rVert_\infty. 
		\end{align*}
	\end{lemma}
	\begin{proof}
	    The Lipschitzness of softmax can be obtained using Proposition 4 in \cite{gao2017properties}. Let us denote the policy under $\softmax{\lambda}{\cdot,Q}$ as $\pi^\lambda_Q$ such that $\pi^\lambda_Q(a | s) = \frac{\exp(\lambda Q(s,a))}{\sum_{a' \in \As} \exp(\lambda Q(s,a'))}$. Now
	    \begin{align}
	        \lVert \softmax{\lambda}{\cdot,Q} - \softmax{\lambda}{\cdot,Q'} \rVert_{TV} & = \lVert \pi^\lambda_Q - \pi^\lambda_{Q'} \rVert_{TV}, \nonumber \\
	        & = \max_{a \in \As} \sum_{s \in \Ss} \big\lvert \pi^\lambda_Q(a | s) - \pi^\lambda_{Q'}(a | s) \big\rvert \label{eq:Lip_soft_inter}
	    \end{align}
	    From Proposition 4 in \cite{gao2017properties}, we know that for any $s \in \Ss$
	    \begin{align}
	        \lVert \pi^\lambda_Q(\cdot | s) - \pi^\lambda_{Q'}(\cdot | s) \rVert_2 = \sqrt{\sum_{a \in \As} \big( \pi^\lambda_Q(a | s) - \pi^\lambda_{Q'}(a | s)  \big)^2} & \leq \lambda \lVert Q(s,\cdot) - Q(s,\cdot) \rVert_2, \nonumber\\
	        & = \lambda \sqrt{\sum_{a \in \As} \big( Q(s,a) - Q'(s,a) \big)^2}, \nonumber \\
	        & \leq \lambda \sqrt{A} \lVert Q(s,\cdot) - Q'(s,\cdot) \rVert_\infty, \nonumber \\
	        & \leq \lambda \sqrt{A} \lVert Q - Q' \rVert_\infty. \label{eq:Lip_soft_inter_2}
	    \end{align}
	    The second inequality is due to the equivalence between $2$ and $\infty$ vector norms. This equivalence also gives us
	    \begin{align}
	         \lVert \pi^\lambda_Q(\cdot | s) - \pi^\lambda_{Q'}(\cdot | s) \rVert_2 = \sqrt{\sum_{a \in \As} \big( \pi^\lambda_Q(a | s) - \pi^\lambda_{Q'}(a | s)  \big)^2} \geq \max_{a \in \As} \big\lvert \pi^\lambda_Q(a | s) - \pi^\lambda_{Q'}(a | s) \big\rvert \label{eq:Lip_soft_inter_3}
	    \end{align}
	    Recalling \eqref{eq:Lip_soft_inter},
	    \begin{align*}
	        \lVert \softmax{\lambda}{\cdot,Q} - \softmax{\lambda}{\cdot,Q'} \rVert_{TV} & = \max_{a \in \As} \sum_{s \in \Ss} \big\lvert \pi^\lambda_Q(a | s) - \pi^\lambda_{Q'}(a | s) \big\rvert \\
	        & \leq \sum_{s \in \Ss} \max_{a \in \As} \big\lvert \pi^\lambda_Q(a | s) - \pi^\lambda_{Q'}(a | s) \big\rvert, \\
	        & \leq \lambda S \sqrt{A} \lVert Q - Q' \rVert_\infty
	    \end{align*}
	    where the last inequality is obtained using \eqref{eq:Lip_soft_inter_2} and \eqref{eq:Lip_soft_inter_3}. 
	\end{proof}
	Now we can further simplify \eqref{eq:e_k_pi} as:
	\begin{align}
		e^{k+1}_\pi & \leq (1-c^{k+1}_\pi) e^k_\pi + \lVert \Delta^k_\pi \rVert_{TV} + c^{k+1}_\pi \lambda S \sqrt{A} \lVert Q^k - Q^*_{\mu^k} \rVert_\infty  
  + d_1 \lVert \mu^{k+1} - \mu^k \rVert_1 + c^{k+1}_\pi \epsilon/2, \nonumber \\
		& \leq (1-c^{k+1}_\pi) e^k_\pi + \lVert \Delta^k_\pi \rVert_{TV} + c^{k+1}_\pi \lambda S \sqrt{A} \epsilon^k_Q 
  + c^{k+1}_\mu d_1 (2  + \lVert \Delta^k_\mu \rVert_1) + \lVert \Delta^k_\mu \rVert_1 + c^{k+1}_\pi \epsilon/2. \label{eq:e_k_pi_2}
	\end{align}
	The first inequality is due to Lemma \ref{lem:softmax_Lip} and the second inequality is due to \eqref{eq:sngl_loop_upd_RL} and the fact that $\lVert \mu\rVert_1 \leq 1$ for any $\mu \in \Ps(\Ss)$. 
	The norms of the drift terms are bounded by
	\begin{align}
		\lVert \Delta^k_\pi \rVert_{TV} \leq c^k_\pi \sum_{t=2}^{T-1} t^{-\zeta} \leq  c^k_\pi \frac{2^{1-\zeta}}{\zeta-1}, \hspace{0.5cm} \lVert \Delta^k_\mu \rVert_1 \leq  c^k_\mu \sum_{t=2}^{T-1} t^{-\zeta} \leq c^k_\mu \frac{2^{1-\zeta}}{\zeta-1}
	\end{align} 
	Rearranging the inequality \eqref{eq:e_k_pi_2},
	\begin{align*}
		& e^k_\pi \leq  \frac{1}{c^{k+1}_\pi} (e^k_\pi - e^{k+1}_\pi) + \frac{2^{1-\zeta}}{\zeta-1} +  \lambda S \sqrt{A}\epsilon^k_Q 
  + \frac{c^{k+1}_\mu}{c^{k+1}_\pi} d_1 \bigg(2 + \frac{2^{1-\zeta}}{\zeta-1} \bigg) +  \frac{c^{k}_\mu}{c^{k+1}_\pi} \frac{2^{1-\zeta}}{\zeta-1} + \frac{\epsilon}{2}, \\
		& \leq  \frac{e^k_\pi - e^{k+1}_\pi}{c^{k+1}_\pi}  +  10 \cdot 2^{1-\zeta} +  \lambda S \sqrt{A} \epsilon^k_Q 
  + 12\frac{c^{k+1}_\mu}{c^{k+1}_\pi} d_1  +  10\frac{c^{k}_\mu}{c^{k+1}_\pi}  + \frac{\epsilon}{2},
	\end{align*}
	for $\zeta \geq 1.1$. Now taking the average over $k = 1,\ldots,K-1$, we get
	\begin{align} \label{eq:error_decomp_avg}
		& \frac{1}{K} \sum_{k=1}^{K-1} e^k_\pi \nonumber \\
		& \leq \frac{1}{K} \sum_{k=1}^{K-1} \bigg( \frac{e^k_\pi - e^{k+1}_\pi}{c^{k+1}_\pi}  +  \lambda S \sqrt{A} \epsilon^k_Q 
  + 12\frac{c^{k+1}_\mu}{c^{k+1}_\pi} d_1 + 10 \frac{c^{k}_\mu}{c^{k+1}_\pi} \bigg) + 10 \cdot 2^{1-\zeta}  + \frac{\epsilon}{2}, \nonumber \\
		& \leq \frac{1}{K} \sum_{k=2}^{K-1}  \Big(\frac{1}{c^{k+1}_\pi} - \frac{1}{c^{k}_\pi} \Big) e^{k+1}_\pi +  \frac{1}{K} \sum_{k=1}^{K-1} \bigg( \lambda S \sqrt{A} \epsilon^k_Q 
  \nonumber \\
		& \hspace{4.4cm}+ (12 d_1 + 20)  \frac{c^{k+1}_\mu}{c^{k+1}_\pi}  \bigg)  + \frac{1}{c^2_\pi K} e^1_\pi - \frac{1}{c^{K+1}_\pi K} e^K_\pi + 10 \cdot 2^{1-\zeta}  + \frac{\epsilon}{2}, \nonumber \\
		& \leq \frac{2}{K} \sum_{k=2}^{K-1} \Big(\frac{1}{c^{k+1}_\pi} - \frac{1}{c^{k}_\pi} \Big) +  \frac{1}{K}\sum_{k=1}^{K-1} \bigg( \lambda S \sqrt{A} \epsilon^k_Q 
  \nonumber \\
		& \hspace{6.6cm}  + (24 d_1 + 40) \frac{c_\mu}{c_\pi} k^{\theta-\gamma} \bigg) + \frac{2}{c^2_\pi K}  + 10 \cdot 2^{1-\zeta} + \frac{\epsilon}{2}, \nonumber \\
		& \leq \frac{2}{K c^K_\pi}   + S \sqrt{A} \lambda \epsilon^k_Q 
  +\frac{ (24 d_1 + 40)c_\mu}{(1+\theta - \gamma)c_\pi} K^{\theta-\gamma}  + \frac{2}{c^2_\pi K}  + 10 \cdot 2^{1-\zeta}  + \frac{\epsilon}{2},
	\end{align}
	where the second to last inequality is due to the fact that $e^k_\pi \leq 2$. 
 Since $\epsilon^k_Q \leq \epsilon_Q / \lambda$, where $\epsilon_Q > 0$, then
	\begin{align} \label{eq:final_error_decomp_pi}
		\frac{1}{K} \sum_{k=1}^{K-1} e^k_\pi & \leq \frac{2}{K c^K_\pi}   + S \sqrt{A}\epsilon_Q 
  +\frac{ (24 d_1 + 40)\bar{\mu}c_\mu}{(1+\theta - \gamma)c_\pi} K^{\theta-\gamma}  + \frac{2}{c^2_\pi K} + 10 \cdot 2^{1-\zeta} + \frac{\epsilon}{2}, \nonumber \\
		& \leq \Os(K^{\theta-1})   + \Os(\epsilon_Q) 
  + \Os (K^{\theta-\gamma})  + \Os(K^{-1})  + \Os(2^{1-\zeta}) + \Os(\epsilon)
	\end{align}
	where 
 $K$ is the total number of episodes.
	
	Now we analyze the mean-field approximation error evolution $e^k_\mu := \lVert \mu^k - \mu^* \rVert_1$. \azedit{Let us define $\hat \mu^{k+1} := (1-c^{k+1}_{\mu}) (\mu^{k} + \Delta^k_\mu) + c^{k+1}_{\mu} (\hat{P}^{k})^\top (\mu^{k} + \Delta^k_\mu)$. Then,}
	\azedit{
	\begin{align*}
		e^{k+1}_\mu & = \lVert \mu^{k+1} - \mu^* \rVert_1 =  \lVert \PP_{S(\epsilon^\net)} [\hat\mu^{k+1},1] - \Gamma_2(\Gamma^\lambda_1(\mu^*),\mu^*)\rVert_1 \\
		& \leq \lVert \PP_{S(\epsilon^\net)}[ \hat\mu^{k+1},1] - \hat\mu^{k+1} \rVert_1 + \lVert \hat\mu^{k+1} - \Gamma_2(\Gamma^\lambda_1(\mu^*),\mu^*) \rVert_1, \\
		& \leq (1-c^{k+1}_\mu) \lVert \mu^k - \mu^* \rVert_1 + (1-c^{k+1}_\mu) \lVert \Delta^k_\mu \rVert_1 + c^{k+1}_\mu \big[  \lVert (\hat P^k)^\top (\mu^k + \Delta^k_\mu) - \Gamma_2(\pi^k,\mu^k) \rVert_1 \\
		& \hspace{6.8cm} + \lVert \Gamma_2(\pi^k,\mu^k) - \Gamma_2(\Gamma^\lambda_1(\mu^*),\mu^*) \rVert_1 \big] + \epsilon^\net , \\
		& \leq (1-c^{k+1}_\mu) e^k_\mu + \lVert \Delta^k_\mu \rVert_1 + c^{k+1}_\mu \big[  \lVert (\hat P^k)^\top (\mu^k + \Delta^k_\mu) - \Gamma_2(\pi^k,\mu^k) \rVert_1 \\
		& \hspace{1cm} + \lVert \Gamma_2(\pi^k,\mu^k) - \Gamma_2(\Gamma^\lambda_1(\mu^k),\mu^k) \rVert_1 + \lVert \Gamma_2(\Gamma^\lambda_1(\mu^k),\mu^k) - \Gamma_2(\Gamma^\lambda_1(\mu^*),\mu^*) \rVert_1\big] + \epsilon^\net,\\
		& \leq (1-c^{k+1}_\mu) e^k_\mu + \lVert \Delta^k_\mu \rVert_1 + c^{k+1}_\mu \big[  \lVert (\hat P^k)^\top (\mu^k + \Delta^k_\mu) - P^\top_{\pi^k,\mu^k} \mu^k \rVert_1  + d_2 \lVert \pi^k - \Gamma^\lambda_1(\mu^k)\rVert_1 \\
		& \hspace{8.2cm} + (d_1 d_2 + d_3)\lVert \mu^k - \mu^* \rVert_1\big] + \epsilon^\net,
		\end{align*}
		\begin{align*}
		e^{k+1}_\mu & \leq (1-c^{k+1}_\mu \bar{d}) e^k_\mu + (1 + c^k_\mu)\lVert \Delta^k_\mu \rVert_1 + c^{k+1}_\mu \lVert (\hat P^k)^\top - P^\top_{\pi^k,\mu^k}\rVert_1  + c^{k+1}_\mu d_2 e^k_\pi + \epsilon^\net, \\
		& \leq (1-c^{k+1}_\mu \bar{d}) e^k_\mu + (1 + c^k_\mu)\lVert \Delta^k_\mu \rVert_1 + c^{k+1}_\mu \sqrt{S} \lVert \hat P^k - P_{\pi^k,\mu^k}\rVert_F  + c^{k+1}_\mu d_2 e^k_\pi + \epsilon^\net, \\
		& \leq (1-c^{k+1}_\mu \bar{d}) e^k_\mu + 11 c^k_\mu 2^{1-\zeta} + c^{k+1}_\mu \sqrt{S} \epsilon^k_P  + c^{k+1}_\mu d_2 e^k_\pi + \epsilon^\net
	\end{align*}}
\noindent where the second to last inequality is due to the equivalence between induced $1$ norm and the Frobenius norm. Rearranging the above inequality,
	\azedit{
	\begin{align*}
		e^k_\mu & \leq  \frac{1}{c^{k+1}_\mu \bar{d}}(e^k_\mu - e^{k+1}_\mu) + 11 \frac{c^k_\mu}{c^{k+1}_\mu \bar{d}} 2^{1-\zeta} +  \frac{\sqrt{S} \epsilon^k_P}{\bar{d}}  +  \frac{d_2 e^k_\pi}{\bar{d}} + \frac{\epsilon^\net}{c^{k+1}_\mu \bar{d}}, \\
		& \leq  \frac{1}{c^{k+1}_\mu \bar{d}}(e^k_\mu - e^{k+1}_\mu) + 22 \frac{2^{1-\zeta}}{\bar{d}} +  \frac{\sqrt{S} \epsilon^k_P}{\bar{d}}  +  \frac{d_2 e^k_\pi}{\bar{d}} + \frac{\epsilon^\net}{c^{k+1}_\mu \bar{d}}
	\end{align*}}
	Taking average over $k = 1,\ldots,K-1$, we get
	\begin{align*}
		& \frac{1}{K} \sum_{k=1}^{K-1} e^k_\mu \leq \frac{1}{K} \sum_{k=1}^{K-1} \bigg[ \frac{1}{c^{k+1}_\mu \bar{d}}(e^k_\mu - e^{k+1}_\mu) +  \frac{\sqrt{S} \epsilon^k_P}{\bar{d}} + \frac{d_2 e^k_\pi}{\bar{d}} + \frac{\epsilon^\net}{c^{k+1}_\mu \bar{d}} \bigg] + 11 \frac{2^{1-\zeta}}{\bar{d}}, \\
		& \leq  \frac{\bar{e}_\mu}{c^{K}_\mu \bar{d} K} + 11 \frac{2^{1-\zeta}}{\bar{d}} + \frac{\sqrt{S}  \epsilon_P}{\bar{d}}  +  \frac{1}{K} \sum_{k=1}^{K-1} \bigg[ \frac{d_2 e^k_\pi}{\bar{d}} + \frac{\epsilon^\net}{c^{k+1}_\mu \bar{d}} \bigg], \\
		& \leq  \frac{\bar{e}_\mu}{c^{K}_\mu \bar{d} K} + 11 \frac{2^{1-\zeta}}{\bar{d}} + \frac{\sqrt{S}  \epsilon_P}{\bar{d}} + \frac{1}{K} \sum_{k=2}^{K} \frac{\epsilon^\net k^\gamma}{c_\mu \bar{d}} +  \frac{1}{K} \sum_{k=1}^{K-1} \frac{d_2 e^k_\pi}{\bar{d}} , \\
		& \leq \Os(K^{\gamma-1}) + \Os(2^{1-\zeta}) + \Os(\epsilon_P) + \Os(K^{\theta-1})   + \Os(\epsilon_Q) 
  + \Os (K^{\theta-\gamma})  + \Os(K^{-1}) + \Os(\epsilon)
	\end{align*}
\noindent where the second inequality is obtained using steps similar to \eqref{eq:error_decomp_avg} and the fact that $\epsilon^k_P \leq \epsilon_P$. The last inequality is obtained using \eqref{eq:final_error_decomp_pi} \azedit{and the fact that $\epsilon^\net \leq c_\mu \bar{d} \epsilon / K^\gamma$}. The proof is thus concluded.
\end{proof}

\subsection{Proof of Corollary \ref{cor:final_bound}}
\begin{proof}
This is a corollary to Theorem \ref{thm:conv_bound}:
	\begin{align*}
		& \bigg\lVert \frac{1}{K} \sum_{k=1}^{K-1} \pi^k - \pi^* \bigg\rVert + \bigg\lVert \frac{1}{K} \sum_{k=1}^{K-1} \mu^k - \mu^* \bigg\rVert \leq \frac{1}{K} \sum_{k=1}^{K-1} \lVert  \pi^k - \pi^* \rVert + \frac{1}{K} \sum_{k=1}^{K-1} \lVert  \mu^k - \mu^* \rVert, \\
		& \leq \frac{1}{K} \sum_{k=1}^{K-1} \lVert  \pi^k - \Gamma^\lambda_1(\mu^k) \rVert + \frac{1}{K} \sum_{k=1}^{K-1} \lVert  \Gamma^\lambda_1(\mu^k) - \pi^* \rVert + \frac{1}{K} \sum_{k=1}^{K-1} \lVert  \mu^k - \mu^* \rVert, \\
		& \leq \frac{1}{K} \sum_{k=1}^{K-1} \lVert  \pi^k - \Gamma^\lambda_1(\mu^k) \rVert  + \frac{d_1 + 1}{K} \sum_{k=1}^{K-1} \lVert  \mu^k - \mu^* \rVert, \\
		& = \Os(K^{\gamma-1}) + \Os(2^{1-\zeta}) + \Os(\epsilon_P) + \Os(K^{\theta-1})   + \Os(\epsilon_Q) + \Os(\epsilon) 
  + \Os (K^{\theta-\gamma})  + \Os(K^{-1}) .
	\end{align*}
	where the last inequality follows from Theorem \ref{thm:conv_bound}.
\end{proof}


%% file: Sec_Lit_Review.tex

Mean-Field Games (MFGs) originated concurrently in the works of \cite{huang2006large,huang2007large} (termed as Nash Certainty Equivalence) and \cite{lasry2006jeux,lasry2007mean} (who coined the term MFG). Since its inception, there have been several works extending the classical concept of MFGs in various directions, such as heterogenous agents \cite{moon2014discrete}, scarce interactions \cite{caines2019graphon,zaman2021adversarial}, risk-sensitive criteria \cite{tembine2013risk,moon2016linear,saldi2020approximate} and cooperative equilibria \cite{bensoussan2018mean,barreiro2021mean}. MFGs have also been applied to a variety of real-world applications such as decentralized charging of EVs \cite{ma2011decentralized}, economics \cite{carmona2020applications} and congestion dynamics \cite{lachapelle2011mean}, among others. Although most of these works have been in the continuous time setting, research in discrete-time MFGs which are much more amenable to Reinforcement Learning have also been gaining momentum recently \cite{saldi2018markov,moon2014discrete}.

RL for MFGs was first dealt with in \cite{guo2019learning} for the finite and in \cite{elie2019approximate} for infinite state and action spaces. The work of \cite{guo2019learning} proposes a double-loop RL algorithm for MFGs with finite state and action spaces MFGs, which involves a projection step onto an $\epsilon$-net. This projection step helps in establishing convergence by utilizing a uniform action gap bound over the $\epsilon$-net. 
A fictitous play algorithm was proposed \cite{elie2019approximate}, involving repeated updates of the mean-field and control policy to approximate the MFE. The first set of works to deal with RL for the benchmark Linear Quadratic (LQ) MFGs were \cite{fu2019actor,zaman2020reinforcement,zaman2022reinforcement}. These works have provided finite sample bounds for the LQ-MFG in the stationary \cite{fu2019actor} and the non-stationary \cite{zaman2020reinforcement,zaman2022reinforcement} settings, by building on policy gradient \cite{fazel2018global} and zero-order stochastic optimization methods \cite{malik2019derivative} for the Linear Quadratic Regulator problem. The recent work of \cite{yongacoglu2022independent} deals with independent learning for a novel setting of $N$-player mean-field games. The work of \cite{lee2021reinforcement} learns the MFE in the special setting of strategic complementarities where a single step of centralized Q-learning followed by a single step of policy execution by many agents is shown to converge to the MFE. The idea of entropy-regularized MFGs was introduced in \cite{xie2021learning,cui2021approximately} along with existence and uniqueness results and RL algorithms to compute the entropy-regularized MFE. The work of \cite{anahtarci2019fitted} also deals with the entropy-regularized MFGs, by utilizing a fitted $Q$-iteration based approach. There have been several works on Deep-RL techniques for MFGs, such as \cite{perrin2021mean,subramanian2019reinforcement}, where \cite{perrin2021mean} uses Deep RL techniques to learn a flocking model observed in nature and \cite{subramanian2019reinforcement} proposes a Neural Network based policy update mechanism. The paper \cite{xie2021learning} proposed a single-loop RL algorithm, such that each critic step leads to a mean-field update as well. This is in contrast to the standard RL for MFG algorithms which have a double-loop structure where multiple critic steps can be executed while keeping the mean-field constant. Our work also has a single-loop structure as each critic step of control policy update leads to a concurrent update of the mean-field. Furthermore, we consider learning along a single sample path of the generic agent without re-initializations.

In addition, the majority of the literature in RL for MFGs assume access to an oracle which can provide the mean-field (or a noisy estimate of it) under a given control policy. This work, on the other hand, proposes the Sandbox learning algorithm which uses the sample path of the agent itself to estimate the mean-field. The two works closest to our setting are \cite{angiuli2022unified} and \cite{yardim2022policy}. The work of \cite{angiuli2022unified} adopts an oracle-less setting but does not provide a finite sample convergence bound of the RL algorithm. Furthermore, the two time-scale update in \cite{angiuli2022unified} updates the $Q$-function at a faster rate whereas Sandbox learning algorithm updates the control policy at the faster rate using a softmax of the estimated $Q$-function. Furthermore, we prove that the episodic nature of learning rates in Sandbox learning is crucial to obtaining finite sample convergence guarantees. Sandbox learning can be extended to entropy-regularized setting by employing a fitted $Q$-iteration (as in \cite{anahtarci2019fitted}). The work of \cite{yardim2022policy} runs independent policy mirror ascent for $N$ agents and proves approximation upto a bias term of $\tilde\Os(1/\sqrt{N})$, with similar resulting sample guarantees. However these guarantees are provided under an assumption that the agent's Markov chain is ergodic under any policy $\pi$, which is stronger than our communicating MDP assumption. This assumption may be a bit restrictive for some scenarios like condestion games with deterministic dynamics. A complete juxtaposition of \cite{angiuli2022unified} and \cite{yardim2022policy} with our work is provided in Section \ref{subsec:rel_lit}.
